\documentclass[twoside]{article}

\usepackage[accepted]{aistats2023}
\usepackage[round]{natbib}

\usepackage{algorithm} 
\usepackage{algpseudocode} 
\usepackage[utf8]{inputenc} 
\usepackage[T1]{fontenc}    
\usepackage{hyperref}       
\usepackage{url}            
\usepackage{booktabs}       
\usepackage{amsfonts}       
\usepackage{nicefrac}       
\usepackage{microtype}      
\usepackage{xcolor}         
\usepackage{placeins}
\usepackage{microtype}
\usepackage{graphicx}
\usepackage{subfigure}
\usepackage{booktabs} 
\usepackage{amsmath}
\usepackage{amsthm}
\usepackage{thmtools, thm-restate}
\DeclareMathOperator*{\argmax}{arg\,max}
\DeclareMathOperator*{\argmin}{arg\,min}
\usepackage{amsfonts}
\newcommand{\nopt}{M}
\newcommand{\ourmethod}{\texttt{ROBOT}}
\newcommand{\ourspluslolbo}{\texttt{LOL-ROBOT}} 
\newcommand{\ourmethodM}{\texttt{ROBOT-$\nopt$}} 
\newcommand{\ourmethodK}{\texttt{ROBOT-$k$}} 
\newcommand{\fullmethodname}{Rank-Ordered Bayesian Optimization with Trust Regions}
\newcommand{\bx}{\mathbf{x}}
\newcommand{\bc}{\mathbf{c}}
\newcommand{\bp}{\mathbf{p}}
\newcommand{\by}{\mathbf{y}}
\newcommand{\bz}{\mathbf{z}}
\newcommand{\bw}{\mathbf{w}}
\newcommand{\inputdom}{\mathcal{X}}
\newcommand{\tr}[1]{\mathcal{T}_{#1}}

\newcommand{\lolbo}{\texttt{LOL-BO}}
\newcommand{\seqscbo}{\texttt{Sequential SCBO}}
\newcommand{\seqlolscbo}{\texttt{Sequential LOL-SCBO}}
\newcommand{\siga}{\texttt{Sitagliptin MPO}}
\newcommand{\valt}{\texttt{Valsartan SMARTS}}
\newcommand{\logp}{\texttt{Penalized Log P}}
\newcommand{\turbo}{\texttt{TuRBO}}
\newcommand{\scbo}{\texttt{SCBO}}
\newcommand{\mapE}{\texttt{CVT-MAP-Elites}}
\newcommand{\divf}{\delta} 
\usepackage{thm-restate}

\begin{document}
\raggedbottom

\twocolumn[
\aistatstitle{Discovering Many Diverse Solutions with Bayesian Optimization}
\aistatsauthor{ Natalie Maus \And Kaiwen Wu \And  David Eriksson \And Jacob Gardner }
\aistatsaddress{ University of Pennsylvania \And  University of Pennsylvania \And Meta \And University of Pennsylvania} 
]

\begin{abstract}
Bayesian optimization (BO) is a popular approach for sample-efficient optimization of black-box objective functions. 
While BO has been successfully applied to a wide range of scientific applications, traditional approaches to single-objective BO only seek to find a single best solution. 
This can be a significant limitation in situations where solutions may later turn out to be intractable. For example, a designed molecule may turn out to violate constraints that can only be reasonably evaluated after the optimization process has concluded. 
To address this issue, we propose Rank-Ordered Bayesian Optimization with Trust-regions (\ourmethod{}) which aims to find a portfolio of high-performing solutions that are diverse according to a user-specified diversity measure. 
We evaluate \ourmethod{} on several real-world applications and show that it can discover large sets of high-performing diverse solutions while requiring few additional function evaluations compared to finding a single best solution.
\end{abstract}

\section{INTRODUCTION}
Bayesian optimization (BO) \cite{jones1998efficient,shahriari2015taking} is a general framework for optimizing black-box functions $\argmin_{\bx^*} f(\bx^{*})$ in a sample-efficient fashion.
BO has been successfully applied to hyperparameter tuning~\cite{snoek2012practical, turner2021bayesian}, A/B testing~\cite{letham2019noisyei}, chemical engineering~\cite{hernandez2017parallel}, drug discovery~\cite{negoescu2011knowledge}, and more.
For example, $f$ may measure the antibiotic activity of a molecule $\bx$, and we might therefore apply BO to design a molecule with high antibiotic activity.

In many of these settings, however, the fact that BO traditionally seeks a single best optimizer $\bx^{*}$ may be a significant limitation. 
This ``all-or-nothing'' attribute of BO is particularly undesirable for problems where the returned $\bx^{*}$ may indeed optimize some useful objective, but is later found to be unsuitable for unforeseen reasons.
For example, a molecule $\bx^{*}$ may have strong \emph{in vitro} antibiotic activity, but may later be found unsafe or ineffective for use in humans through clinical testing. 
Worse, incorporating constraints like human safety directly into the optimization procedure of unknown molecules seems intractably expensive at best and unethical at worst.

In these settings where the risk of wasting the evaluation budget in search of $\bx^{*}$ is high, practitioners benefit from being given--in addition to the best single optimizer we can find--a series of alternative solutions to the problem: a set of ``back up plans.'' Formally, we might seek a set $S^{*} = \left\{\bx^{*}_{1}, \bx^{*}_{2}, ..., \bx^{*}_{\nopt}\right\}$ of solutions that are all of high objective value, but we may require that these solutions are sufficiently \emph{diverse} to ensure that this set of solutions is less likely to later fail for unrelated reasons. 
The practitioner may therefore provide us with a \emph{symmetric diversity measure} $\divf(\bx, \bx')$ that must exceed some threshold $\tau$ for all pairs of solutions in the set $S^{*}$.
For example, a biochemist may require that molecules in $S^{*}$ have sufficiently low fingerprint similarity~\cite{GuacaMol}. 
Solving this problem efficiently equips practitioners with large sets of \emph{potential} solutions to their true problem, which can mitigate the risk that a single optimizer fails to be useful. 
This diverse solution search problem is challenging to cast under existing BO frameworks. 
The diversity constraints $\divf(\bx^{*}_{i}, \bx^{*}_{j}) \geq \tau$ are challenging to view as traditional black-box constraints, as the constraint functions for the $j$th point in $S^{*}$ depend on the locations of the other points in $S^{*}$.

In this paper, we propose \ourmethod{}, a method to solve the above problem and find a diverse set of high scoring solutions $S^{*}$ so that $\forall \bx^{*}_{i},\bx^{*}_{j} \in S^{*}, \,\delta(\bx^{*}_{i}, \bx^{*}_{j}) \geq \tau$. Across a variety of problem settings ranging from reinforcement learning to molecule design, \ourmethod{} is able to discover large sets of diverse solutions $S^{*}$ with little loss in efficiency over finding a single best solution.

\textbf{Contributions}

\begin{enumerate}
    \item We introduce the problem setting of finding a set of high-performing solutions under a \textit{user-defined diversity measure}. While prior work outside the BO literature has considered a similar setting where diversity is taken to be distance in input space (e.g. \cite{CV-MAP-Elites,MAP-Elites}), this is the first work we are aware of to consider arbitrary, user-defined diversity measures like fingerprint similarity for molecules.
    \item We propose a local Bayesian optimization solution, \ourmethod{}, to this problem that extends to large sample sizes, high dimensional inputs, and structured inputs.
    \item We provide empirical results across challenging, high-dimensional optimization tasks to show that our algorithm can consistently produce large populations of diverse, high-preforming solutions with \textit{virtually no loss in efficiency} compared to finding a \textit{single} solution. 
    \item We additionally demonstrate results on structured drug discovery tasks using the widely used fingerprint similarity function as a diversity measure, demonstrating the value of our approach to practitioners in the physical sciences.
    \item We prove global consistency of \ourmethod{} in \autoref{sec:global_convergence}.
    \item We release an open-source implementation of \ourmethod{} using BoTorch~\citep{balandat2020botorch}.
\end{enumerate}
\section{BACKGROUND AND RELATED WORK}
\label{sec:background}

\paragraph{Bayesian optimization.} 
Bayesian optimization (BO) \cite{movckus1975bayesian, SnoekBO} is an approach to sample-efficient black-box optimization that utilizes a probabilistic \textit{surrogate model}--commonly a Gaussian process (GP) \cite{rasmussen2003gaussian}--and an \textit{acquisition function} that leverages the surrogate to find the most promising candidates to evaluate next. 
BO is a sequential optimization algorithm that proceeds in iterations. 
In each iteration, a surrogate model is trained on data collected from evaluating the black-box objective function. 
The acquisition function, defined given the surrogate model's predictive posterior, is then maximized to select one or more candidates to evaluate next, trading off exploration and exploitation.

\paragraph{Parametric Gaussian process regressors (PPGPR).} 
Because we consider tasks with large function evaluation budgets, we use an approximate GP surrogate model. Approximate GP models use inducing point methods in combination with variational inference to allow approximate GP inference on large data sets \cite{hensman2013gaussian,titsias2009variational}. In this work, we use the PPGPR approximate GP model proposed by \citet{PPGPR}, which we found provides substantial improvements in Bayesian optimization performance on the molecule optimization tasks we consider.

\paragraph{Constrained Bayesian optimization.} 
While it may be tempting to attempt to formulate our problem as a constrained BO problem \cite{cei,pesc,scbo}, this is challenging as the constraints depend on the set $S^*$ and are therefore non-stationary. 
In particular, the $i$th point in $S^{*}$, $\bx^{*}_{i}$, must satisfy $i - 1$ constraints that depend on $\bx^{*}_{1},...,\bx^{*}_{i-1}$ which are unknown in advance.
One potential solution is to acquire the points in $S^{*}$ sequentially.
Specifically, unconstrained optimization can be used to obtain $\bx^{*}_{1}$, the second point can then be acquired subject to the single constraint that $d(\bx^{*}_{1}, \bx^{*}_{2}) \geq \tau$, and so on. 
As a baseline, we adapt the work of \cite{scbo} to utilize this strategy, which we refer to as \seqscbo{} in \autoref{sec:experiments}.

\paragraph{Multi-objective Bayesian optimization (MOBO).}
There has been much work in developing new methods for MOBO in recent years \cite{mobo1, mobo2, mobo3, morbo, lambo}. However, these methods cannot be readily applied since our problem setting is quite different. 
Diversity in our setting is not a second objective since we do not try to maximize the diversity, but instead we require that the diversity between pairs of solutions exceed some threshold $\tau$. 
For example, if two molecules are sufficiently diverse, they can be expected to have relatively unrelated chemical properties and further increasing their diversity may not provide much value to the practitioner. 
If one desires to simultaneously maximize the diversity between solutions, this becomes a very different multi-objective problem. 
Using existing MOBO methods for even this different problem setting is itself non-trivial because any diversity measure cannot be measured for a single point in isolation and the Pareto frontier therefore does not exist in the usual sense. 
This prevents the direct application of methods such as that of \citet{R5} which seek to maximize diversity along the Pareto frontier.

\paragraph{Generative modeling for molecules.}
Many generative modeling approaches have been proposed to generate populations novel molecules. This includes variational autoencoder models such as the Junction Tree Variational Auto Encoder \cite{JTVAE} and the SELFIES-VAE \cite{maus2022local}. 
Populations of molecules generated by sampling from these models can be evaluated for diversity and validity using methods such as \citet{R3-1}. 
However, while these models can successfully generate diverse populations of molecules, they are not designed to generate molecules with any particular user specified characteristics. Thus, when we desire a diverse population of molecules which also each have some set of desirable traits, it becomes necessary to use black box optimization tools on top of these generative models. 

\paragraph{Bayesian optimization for molecular design.}
BO has been utilized extensively in recent years for molecular design problems, both over fixed pre-defined lists of existing molecules~\citep{Williams2015-cp, hernandez2017parallel, Graff2021-wr} and by utilizing the continuous latent spaces of variational autoencoders~\citep{gomez2018automatic,eissman2018bayesian,Weighted_Retraining,Huawei,siivola2021good,JTVAE}. 
In latent space optimization, an encoder $\Phi(\bx)$ is used to map molecules $\bx$ to real-valued latent vectors $\bz$. 
BO is then applied in the continuous latent space, and candidate latent vectors are decoded using the decoder $\Gamma(\bz)$ to generate candidate molecules. 
\citet{maus2022local} introduce an extension of \turbo{} \cite{TuRBO} to the latent space optimization setting where the surrogate model and VAE are trained jointly using variational inference -- in our molecular design results in \autoref{sec:experiments} we will make this same adaptation for our method.

%
\paragraph{Population generation algorithms.}
Although some general frameworks have been proposed to extend Bayesian optimization techniques to problems outside of optimization \cite{R3-2}, to the best of our knowledge, this work is the first to consider the setting of finding multiple solutions under user specified diversity constraints using Bayesian optimization. 
We therefore compare to population generation algorithms designed to generate populations of solutions, some of which are designed for ``quality diversity'' \cite{CMAES, MAP-Elites, BOP-ELITES, CV-MAP-Elites, Multitask-MAP-ELITES, gaier2018data, R2}. 
Most relevant is the \mapE{} algorithm \cite{CV-MAP-Elites} which extends much of this work to high-dimensional optimization tasks by avoiding constructing exponential discretiziations as in e.g. \cite{MAP-Elites,BOP-ELITES,gaier2018data}. 
However, these algorithms measure diversity via distance in the search space, and are not straightforward to adapt to user-specified notions of diversity. 
In \autoref{sec:experiments} we show these approaches can fail to find diverse solutions for many optimization tasks when using semantically meaningful notions of diversity.
\vspace{-1ex}

%
\paragraph{Trust Region Bayesian Optimization (TuRBO)} Traditional approaches to BO are generally limited to low-dimensional problems with at most twenty tunable parameters~\cite{frazier2018tutorial}. 
Many methods tailored for high-dimensional BO are only suitable for small evaluation budgets and generally make strong assumptions on the underlying black-box objective function~\cite{kandasamy2015high,wang2016rembo,letham2019re,mutny2019efficient}.

\citet{TuRBO} proposed \turbo{}-$M$ which maintains $M$ local optimization runs, each with its own dataset $\mathcal{D}_{i}$ and surrogate model. Each local optimizer $i$ proposes candidates from within a hyper-rectangular \textit{trust region} $\tr{i}$ and a batch of candidates is selected from across all local optimizers in each iteration. Because acquisition is performed globally across all trust regions, local optimizers with the most promising evaluations of the objective receive a larger fraction of the evaluation budget. Each trust region $\tr{i}$ is a rectangular subset of the input space $\mathcal{X}$ centered at the best point found by the $i$th local optimizer--the \textit{incumbent}--$\bx^{+}_i$ and has a side-length $\ell_i \in [\ell_{min}, \ell_{max}]$. If a local optimizer improves upon its own incumbent $\rho_{succ}$ times in a row, $\ell_i$ is increased to $\min(2\ell_i, \ell_{max})$. Similarly, when a local optimizer fails to make progress $\rho_{fail}$ times in a row, the length $\ell_i$ is reduced to $\ell/2$. If $\ell_i < \ell_{min}$, that local optimizer is restarted. 
While \turbo{}-$M$ is not directly applicable to our problem setting, we will also use multiple trust regions with the same length adjustment dynamics. 
Additionally, while \turbo{}-$M$ keeps a separate data history for each trust region, other trust region methods such as \texttt{MORBO} \cite{morbo} allow data sharing such that trust regions can be recentered on candidates from the data history of other trust regions.

\section{METHODS}
\label{sec:methods}
We consider the task of finding a diverse set of $\nopt$ solutions for some high-dimensional objective function $f(\cdot)$. 
For a given input $\bx$, we can evaluate $f(\bx)$ to obtain a (possibly noisy) objective value $y$. 
To measure diversity, we use a symmetric, user-supplied function $\divf{}(\bx_1, \bx_2)$ defined on pairs of points in the search space $\inputdom$. Formally, we seek a sequence $S^{*} := \left\{ \bx^{*}_{1}, \ldots, \bx^{*}_{\nopt} \right\}$ so that:
\begin{align}
    \bx^{*}_{1} &= \argmax_{\bx \in \inputdom} f(\bx) \label{eq:problem_def}\\
    \bx^{*}_{i} &= \argmax_{\bx \in \inputdom} f(\bx)  \; \mathrm{s.t.} \; \delta(\bx_{i}^{*}, \bx_{j}^{*}) \geq \tau \text{ for } j=1,\ldots, i-1 \nonumber 
\end{align}
Under this formalization of the problem setting, the optima in  $S^{*}$ are defined \textit{hierarchically}. In particular, we explicitly still wish to recover the best possible optimizer $\bx^{*}_{1}$ of the original objective. This choice of formalization is deliberate: our goal in this paper is to develop a method that still optimizes the given objective function $f(\cdot)$ as a practitioner may expect, but also produce alternative high quality solutions that are meaningfully different to the optimum as a by-product with as few additional evaluations as possible. 

\subsection{\fullmethodname{} (\ourmethod)} 
\label{sec:ours} 
In this section, we propose \ourmethod{} - an algorithm which extends Bayesian optimization to the problem setting above. We demonstrate the global consistency of our approach in \autoref{sec:global_convergence}. In order to find a set of $\nopt$ solutions, \ourmethod{} maintains $\nopt$ simultaneous local optimization runs using $\nopt$ individual trust regions. As in prior work, trust regions are defined as rectangular regions of the input space, e.g. $\tr{i} \subseteq \mathcal{X}$, with side lengths defined using standard Euclidean distance. We note that it would be interesting to explore the setting where trust region side lengths are instead defined with respect to the diversity measure $\delta(\cdot, \cdot)$; however, the problem of even sampling candidate $\bx$ locations from within such a trust region becomes challenging in the general setting.

Each local run $i$ aims to find a single diverse solution $\bx^{*}_{i}$, which together form the desired set $S^{*}$. In our problem definition, the solution $\bx^{*}_{i}$ is only constrained with respect to prior solutions, e.g. for $j < i$.
Mirroring this, we assign a hierarchical rank-ordering to the $\nopt$ trust regions, $\tr{1}, \tr{2}, ..., \tr{\nopt}$, so that the local optimization run $\tr{i}$ responsible for finding $\bx^{*}_{i}$ is only constrained by local optimizers $\tr{j}$ with $j < i$.

\paragraph{Acquiring candidates.}
Because the diversity constraints are non-stationary, they must be handled in an online fashion as we explore the input space. 
A natural way to accomplish this is to enforce diversity with respect to all candidates chosen by the optimization procedure. 
Mirroring the optimization problem in \autoref{eq:problem_def}, in each iteration of optimization we select candidates $\hat{\bx}_i$ from each trust region $\tr{i}$ to improve over its own incumbent $(\bx_{i}^{+}, y_{i}^{+})$ using a similarly hierarchically-constrained acquisition function:
\begin{align}
\hat{\bx}_1 &= \argmax_{\bx \in \tr{1}} \alpha (\bx;y^{+}_{1}) \nonumber \\
\hat{\bx}_i &= \argmax_{\bx \in \tr{i}} \alpha (\bx;y^{+}_{i})  \; \mathrm{s.t.} \; \delta(\bx, \hat{\bx}_j) \geq \tau \; \forall j < i
\end{align}
Here, $\alpha$ may be a standard maximization acquisition function such as Expected Improvement (EI), and $y^{+}_{i}$  refers to the best objective value observed so far by the $i$th optimizer.
By asymmetrically constraining candidates, we select diverse sets of candidates. Furthermore, because high ranking trust regions $\tr{i}$ are less constrained they are unimpeded by lower-rank trust regions $\tr{j}$ where $i < j$. 
As a consequence of this, the highest ranking trust region, $\tr{1}$, is never impeded by any other trust region. For an illustration of this, see \autoref{fig:diagram}. 
\begin{figure*}[!ht]
    \vskip -0.1in
    \begin{center}
        \centerline{\includegraphics[width=0.68\textwidth]{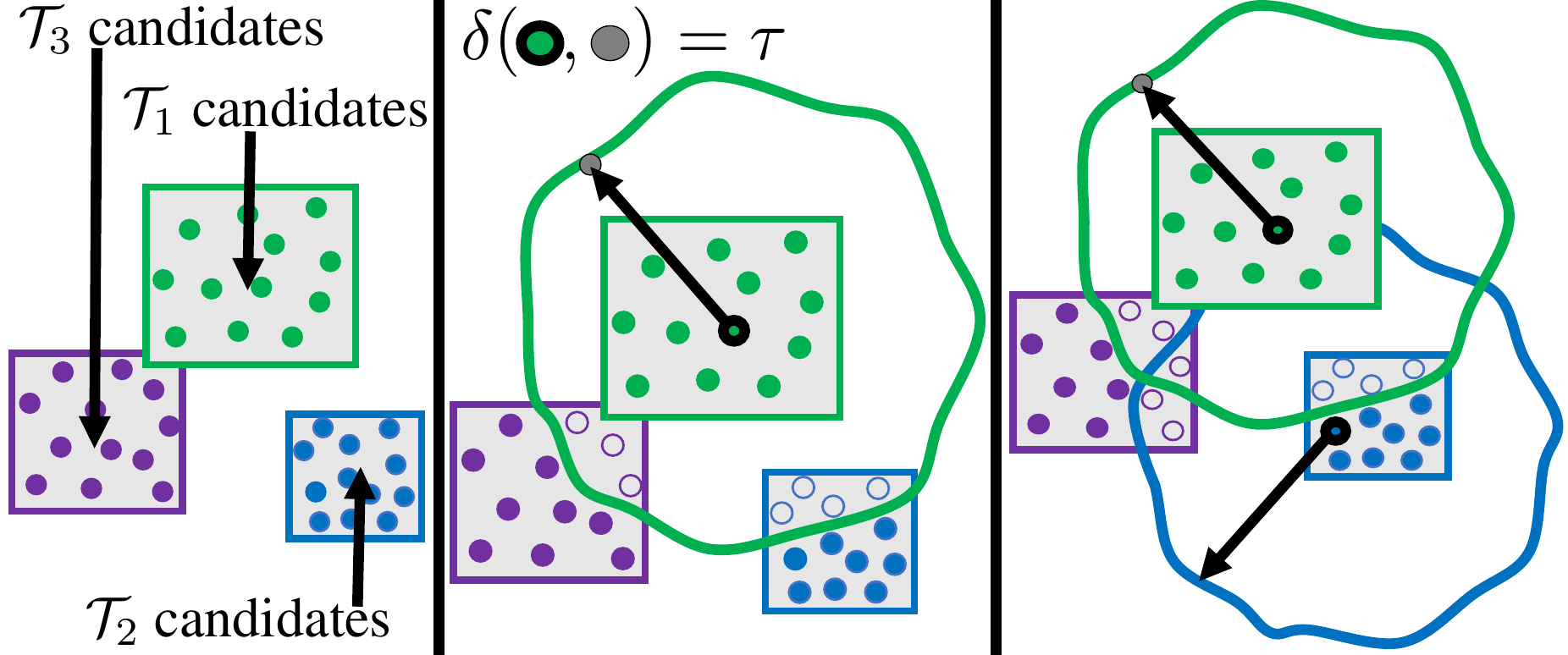}}
        \caption{
            Diagram of three rank-ordered subordinate trust regions, Green $\succ$ Blue $\succ$ Purple. 
            \textbf{(Left)} Each trust region generates $b$ candidates. 
            \textbf{(Middle)} Starting with $\tr{1}$, we discard candidates in subordinate trust regions that violate diversity constraints with candidates in $\tr{1}$. 
            \textbf{(Right)} We repeat this procedure with $\tr{2}$, removing infeasible candidates from $\tr{3}$.}
        \label{fig:diagram} 
    \end{center}
    \vskip -0.2in
\end{figure*}
Since $\delta$ is an arbitrary user-defined black-box function, the above optimization problem is challenging. However, when the acquisition function maximization is approximated via a discretization of the input space $\mathcal{X}$--a relatively common approach--the above optimization remains straightforward. Nevertheless, a reliance on discretization motivates the usage of a modified Thompson sampling procedure that we describe below.

\paragraph{Thompson Sampling} 
Another approach to acquisition is to use hierarchically constrained Thompson Sampling (TS). To accomplish this, we select a candidate $\hat{\bx}_i$ for each trust region $\tr{i}$ one at a time, in rank-order. 
To select a candidate $\hat{\bx}_i$ from $\tr{i}$, we sample $r$ points $C_{i} = \left\{\bc_{i1}, \bc_{i2}, ..., \bc_{ir}\right\}$ from $\tr{i}$.
We then sample a realization $\hat{f}(\bc_{ij})$ for each of these.
Denote by $P_{i}$ the set of all candidate points which have already been selected from each of the higher-ranking trust regions $\tr{1}, \tr{2}, ..., \tr{i-1}$. We select a batch of candidates from among those points in $C_{i}$ that are feasible with respect to all points in $P_{i}$:
\begin{equation*}
    \hat{F} = \{ \bc \mid \bc \in C_{i} , \forall \bp \in P_{i} \; \delta(\bc, \bp) \geq \tau   \}
\end{equation*}
If $\hat{F}$ is empty, then no point in the discretization $C_{i}$ of the interior of $\tr{i}$ was feasible, and we select no candidate from $\tr{i}$ in this round. Because our experiments are run mostly on high dimensional settings where the trust regions are separated by relatively large distances, we found this to be an extremely rare occurrence empirically, happening only a handful of times across all experiments. 

\paragraph{Trust region modifications.}
In each iteration, \ourmethod{} recenters the trust regions such that the current set $S^{+}_{t}$ after iteration $t$ approximating $S^{*}$ is equivalent to the set of all centers $\left\{\bx^{+}_{1},...,\bx^{+}_{\nopt}\right\}$. 
When trust regions select new centers, all diversity constraints change, and since higher-ranking trust regions are unconstrained by subordinate ones, they may re-center on candidates that cause subordinate trust region incumbents to violate these new constraints. This would render some points in $S^{+}_{t}$ suddenly infeasible. 
To remedy this, in each step $t$, we greedily reconstruct the feasible set $S^{+}_{t}$ using the full set of data points $D_{t}$ evaluated so far by all optimizers. In particular, we set $S^{+}_{t}=\{\bx'^{(t)}_{1},...,\bx'^{(t)}_{\nopt}\}$, where:
\begin{align}
    \bx'^{(t)}_{1} &= \argmax_{(\bx,y) \in D_{t}} y \nonumber \\
    \bx'^{(t)}_{i} &= \argmax_{(\bx,y) \in D_{t}} y \;\; \mathrm{s.t.} \;\; \forall j < i \;\; \divf{}(\bx, \bx'^{(t)}_{j}) \geq \tau 
\end{align}
We then re-center trust region $\tr{i}$ on $\bx'^{(t)}_{i}$. 
As a result, $\tr{1}$ is always centered on the best-scoring point found by any trust region, $\tr{2}$  is centered on the best remaining point which is sufficiently diverse from the new center of $\tr{1}$, and so on.
In addition to recentering, we note that optimizer $i$ is only attempting to improve on its own current incumbent objective value (subject to its own diversity constraints), and not trying to improve over the globally best value observed so far (e.g. $y^{+}_{1}$). Therefore, trust region successes and failures as described in \autoref{sec:background} are defined with respect to each trust region's own incumbent.

\paragraph{Global surrogate model}
The recentering procedure described above can recenter a trust region $\tr{i}$ on any data point in the entire optimization history $D_{t}$, not merely its own local optimization history. This makes the use of local GP surrogate models for each trust region ill suited to the task, as $\tr{i}$ may move to locations that were not acquired by its own local surrogate. $\ourmethod$ therefore instead maintains a single global surrogate model across all $\nopt$ trust regions. The benefit of this is isolated in \autoref{fig:valt_ablation}.

\section{EXPERIMENTS}
\label{sec:experiments}
We apply \ourmethod{} to five high-dimensional BO tasks for which finding a diverse set of solutions is desirable. 
We additionally optimize diverse S\&P $500$ investment portfolios in the appendix.
Three of these tasks are continuous problems that enable direct application of \ourmethod{} as described in \autoref{sec:methods}. 
The last three are structured molecule optimization tasks.

\paragraph{Implementation details and hyperparameters.} 
We implement \ourmethod{} leveraging BoTorch~\cite{balandat2020botorch} and GPyTorch~\cite{gardner2018gpytorch}, with code available at \url{https://github.com/nataliemaus/robot}.
All trust region hyperparameters are set to the \turbo{} defaults as used in \citet{TuRBO}. 
Particular choices of new task-specific parameters, $\nopt$, $\tau$, and $\divf{}(\cdot)$ are motivated in the corresponding section for each task. 
Since we consider large numbers of function evaluations for several tasks, we use an approximate GP surrogate model. 
In particular, we use a PPGPR for all tasks \cite{PPGPR}. 
The number of random points used to initialize optimization is kept consistent across all methods compared for each task and is included in the x-axis of all plots.
We use Thompson sampling for all experiments.
See \autoref{sec:detials} for additional implementation details. 

\paragraph{Extending \ourmethod{} to the structured BO setting.} 

To extend \ourmethod{} to the structured setting for the molecule optimization tasks, we use the pre-trained SELFIES VAE introduced by \citet{maus2022local} to map from the structured molecule space to a continuous search space where Bayesian optimization can be directly applied. 
Additionally, as in \texttt{LOL-BO} \citep{maus2022local}, we periodically train the surrogate model jointly with the VAE end-to-end in order to encourage the continuous latent space to be organized in a way that is more amenable to optimization.

We refer to this extension of \ourmethod{} as \ourspluslolbo{}.

\paragraph{Plots.} 
For each task, we plot the objective value of the current feasible set of $\nopt$ solutions obtained after a certain number of function evaluations. 
All plots with $\nopt = 1$ show the objective value of the single best solution found, and are included to highlight the loss in efficiency incurred by all methods when seeking larger sets of solutions. 
For baselines such as Standard BO, \turbo{}, and \turbo-$M$ which are designed to find a single solution rather than a diverse set, we plot the mean objective value of the best $\nopt$ feasible solutions found in the history of the run. 

All plots are averaged over multiple runs and show standard errors. 
The expensive Guacamol experiments used $10$ repetitions, while all others used $20$. On many plots we include runs of \ourmethodK{}---i.e., our method seeking $k$ solutions---for $k > \nopt$. While one would in practice always run \ourmethod-$\nopt$ when seeking $\nopt$ solutions, these results demonstrate the loss of efficiency of discovering $\nopt$ solutions when seeking more.
For plots where we show figures with different $M$, we do not plot \ourmethodK{} for $k < M$. 

\paragraph{Baselines.} 
In all plots, we compare \ourmethod{} against \turbo{}, \turbo-$M$, and an alternative diverse optimization algorithm involving sequential runs of constrained \turbo{} (see \texttt{Sequential SCBO} description below). 
Since these algorithms are variants of \turbo{}, each can be adapted to the latent space BO setting using an analogous version of \texttt{LOL-BO}.

We denote these latent space adaptations using notation: \lolbo{}, \lolbo{}-$M$, and \seqlolscbo{}. 
Applying end-to-end training to each baseline allows them to be directly compared to \ourspluslolbo{} for the three molecular tasks. 

Note that when $M=1$, \turbo{}-$M$, and \texttt{Sequential SCBO} are the same algorithm so we only plot \turbo{}. 

Although baselines such as \texttt{Standard BO}, \turbo{}, and \turbo{}-$M$ are designed to find a single solution rather than $M$ diverse solutions, we compare to them in plots with $M > 1$ by plotting the mean of the best $M$ diverse solutions found by the method along the way. 

\paragraph{Sequential \scbo{} Baseline} 
As discussed in \autoref{sec:background}, we can cast our problems as a constrained optimization problem if the solutions in $S^{*}$ are generated sequentially--e.g., the constraints for $\bx^{*}_{2}$ are well defined given a fixed $\bx^{*}_{1}$. 
To directly compare against this alternative, we run \scbo{} $\nopt$ times in a row, where the $i$th run of \scbo{} has diversity constraints against all solutions found from runs $j < i$. 
We additionally make several modifications to improve the efficiency of this algorithm. We start each sequential run from the best point observed on any previous run which meets the new set of constraints. Additionally, we maintain the same surrogate model across sequential runs rather than discard data.

\subsection{Continuous BO Tasks}
\label{sec:continuous bo}
In this section, we consider two optimization tasks for which finding a diverse set of solutions is useful--optimizing the trajectory of a mars rover, and optimizing the policy used by a lunar landing device. 

\subsubsection{Rover}
The rover trajectory optimization task consists of finding a $60$ dimensional policy that allows a rover to move along some trajectory while avoiding a set of obstacles~\cite{ebo}. 
This optimization problem is useful as it allows us to directly visualize the diverse paths found by optimization.

\begin{figure*}[!ht]
    \vskip -0.2in
    \begin{center}
        \centerline{\includegraphics[width=\textwidth]{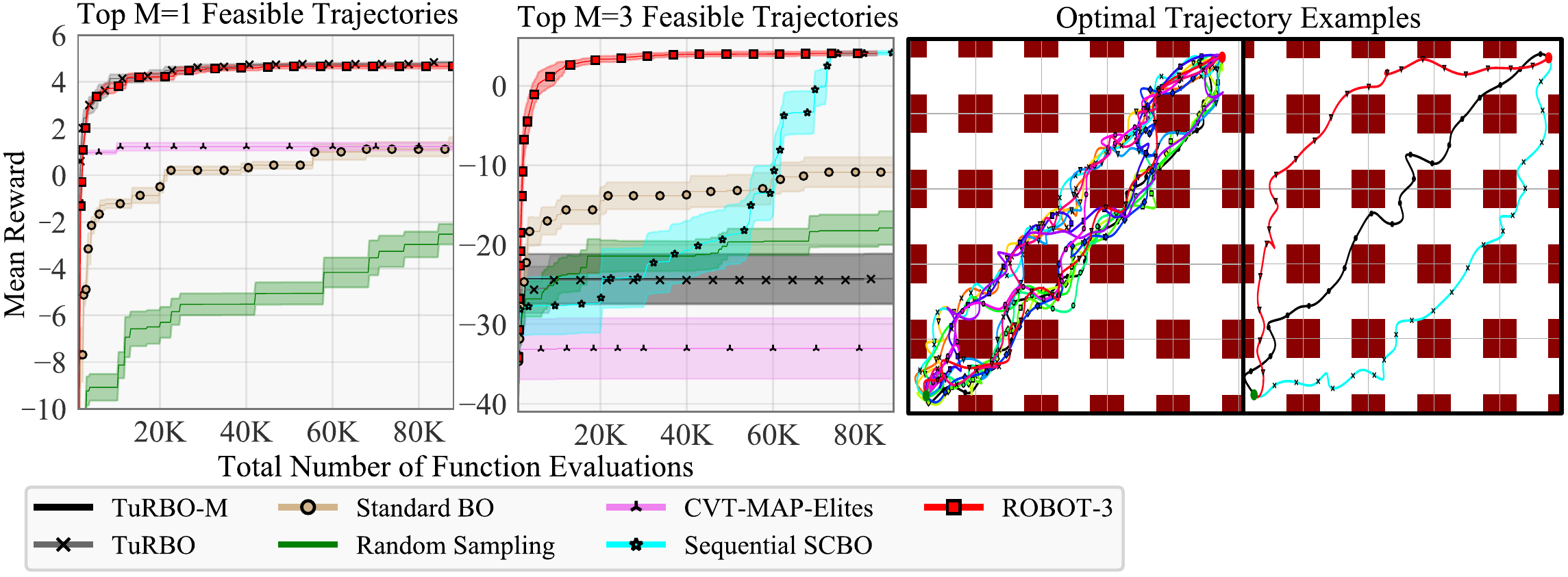}}
        \caption{
            \textbf{(Left, Middle Left)} Rover optimization results where feasible trajectories have a minimum one-way-distance (OWD) of $0.15$.
            Note that in the leftmost figure, \turbo{} and \ourmethod{} preform similarly so the curves overlap and are difficult to differentiate. Similarly for \turbo{} and \turbo{}-$M$ in the middle left figure.
            \textbf{(Middle Right)} $15$ optimized trajectories found by $15$ individual runs of \turbo{} (without diversity constraints).
            \textbf{(Right)} Three diverse trajectories found by \ourmethod{}.}
        \label{fig:rover} 
    \end{center}
    \vskip -0.2in
\end{figure*}

\paragraph{Diversity function $\divf{}$ and threshold $\tau$}
A meaningful diversity measure is one that requires the resulting trajectories to take distinct routes around the obstacles. 
To measure the distance between two trajectories, we use the one-way-distance ($\divf{}_{OWD}$) metric from \cite{owd}.
The obstacles used in the rover environment are squares of side length $0.05$ and all four sides of each obstacle are $0.1$ from the side of some other obstacle (see \autoref{fig:rover}). 
Since our goal is to find diverse trajectories which take different routes around the obstacles, we therefore set $\tau = 0.15$.

\paragraph{Results.} 
Results on this task for $\nopt = 3$ trajectories are displayed in \autoref{fig:rover}. 
The leftmost figure depicts convergence speed of the top trajectory optimized only. \ourmethodM{} incurs no decrease in efficiency for finding the best trajectory for $M=3$, desite also finding (middle left) three diverse trajectories of equivalent reward to the first. 
Although Standard BO fails to find a good single solution (leftmost figure), it outperforms \turbo{} when finding $M > 1$ solutions (middle left figure). Standard BO is less myopic, and therefore finds a larger number of diverse trajectories with positive reward. 
\mapE{} performs worse than all other baselines when we take the average of the top three diverse solutions found (middle left figure). Likely, this is due to the usage of input space distance as a diversity measure, which does not necessarily correlate with the more semantically meaningful diversity measure $\divf{}$. 
Trajectories found by multiple runs of \turbo{} and a single run of \ourmethod{} are depicted in the middle right and rightmost plots, clearly demonstrating diversity. 

\subsubsection{Learning Robust Lunar Lander Policies}
The lunar lander optimization task seeks a control policy that enables an autonomous lander to land without crashing on a randomly generated set of terrain environments. 
The objective function is defined as the average reward of the policy obtained on a set of environments. 
We optimize the same controller as in~\citep{TuRBO}. 
Although \turbo{} finds policies that land on the training environments, we find that these policies sometimes crash when tested on unseen environments.
\paragraph{Diversity function $\divf{}$ and threshold $\tau$.}
Since there is no obvious semantically meaningful measure of diversity between two policies for this task, we define $\divf{}$ to be the euclidean distance between two policies. 
We choose $\tau = 0.6$ since, for larger values of $\tau$, the random set of 1024 policies used to initialize optimization often did not contain a sufficient number of feasible policies to start from. 
\paragraph{Results.} 
To demonstrate the value of this notion of diversity in this setting, we use \ourmethod{} to find a diverse set of 20 policies $S^{*}$ and then construct a single robust policy which simply takes the majority vote action of the diverse policies $\bx^*_i$ at every step. 
For comparison, we generate twenty policies by running \turbo{} sequentially twenty times (requiring $20 \times$ as many evaluations). 
In \autoref{fig:lunar} (right), we plot a histogram of rewards obtained by each of these strategies on 200 unseen environments. 
Without diversity constraints, the policies obtained by \turbo{} occasionally achieve catastrophically low rewards. 
However, the ensembled 20 diverse policies never fail to land across this larger set of environments.

To demonstrate optimization efficiency, we plot function evaluations versus mean objective value found in \autoref{fig:lunar}. 
We show results for optimizing a set of $\nopt = 1$ and $20$ feasible policies. Despite distributing its evaluation budget to find twenty diverse solutions, \ourmethod{} incurs only a $3 \times$ slowdown. 
Although $\mapE{}$ converges faster, $\ourmethod{}$ eventually obtains a higher mean objective value. This task is particularly well suited for \mapE{} as it uses input space diversity measures. 
This task is one of the least successful for \ourmethod{}, as discovering 20 policies requires roughly $6\times$ as many evaluations, where as most tasks in this section require significantly fewer additional evaluations. 
Nevertheless, this is still better than linear slowdown.

\begin{figure*}[!ht]
    \vskip 0.2in
    \begin{center}
        \centerline{\includegraphics[width=0.96\textwidth]{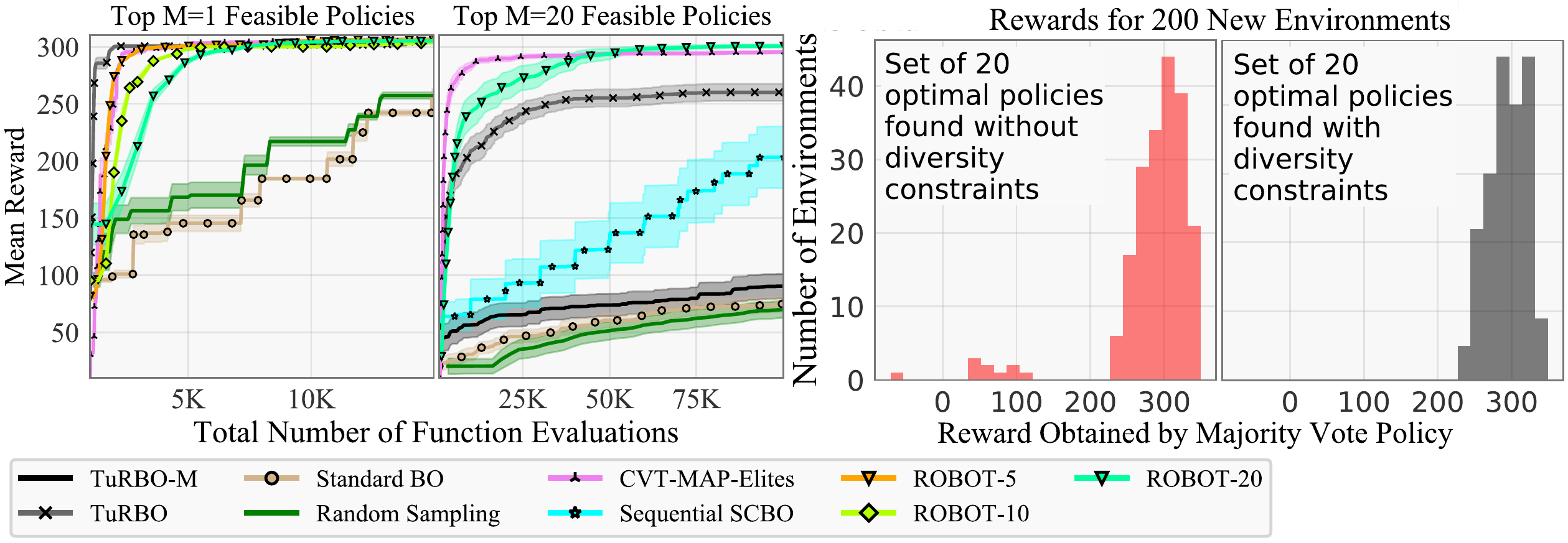}}
        \caption{\textbf{(Left, Middle Left)} Lunar lander optimization, feasible policies have a minimum Euclidean distance of 0.6 in parameter space.
        Note that in the leftmost figure, \texttt{CVT-MAP-ELITES} and \ourmethod{}-$5$ preform similarly so the curves overlap and are difficult to differentiate. \textbf{(Middle Right, Right)} Rewards obtained by majority vote ensembling $20$ policies for $200$ new environments. \textbf{(Middle Right)} Set of policies obtained by $20$ independent runs of \turbo{}. \textbf{(Right)} Set of policies obtained by a single run of \ourmethod{}-$20$. }
        \label{fig:lunar} 
    \end{center}
    \vskip -0.2in
\end{figure*}

\subsection{Molecular BO Tasks}
\label{sec:mols}
The Guacamol benchmark suite~\cite{GuacaMol} contains scoring oracles for a variety of molecule design tasks, with scores ranging between 0 and 1. 
Of these tasks, we select two, \siga{} and \valt{}, for which high-scoring molecules found by \lolbo{} tended to have high fingerprint similarity, making a search for diverse solutions particularly desirable. The task definitions are discussed in \citet{GuacaMol}. 
Because it is arguably the most commonly studied molecule optimization problem, we include \logp{} results in the appendix.

\paragraph{Additional baseline.} 
In the molecule design setting, additional methods outside of BO have been proposed that specifically produce populations of solutions. 
While these populations are not specifically constrained to be diverse in any way, the Guacamol scoring procedure often scores the entire population generated. 
Therefore, we additionally compare to \texttt{Graph GA} \cite{graphGA}, one of the top performing methods on the Guacamol benchmark leaderboard.

\begin{figure*}[!ht]
    \vskip 0.2in
    \begin{center}
        \centerline{\includegraphics[width=0.96\textwidth]{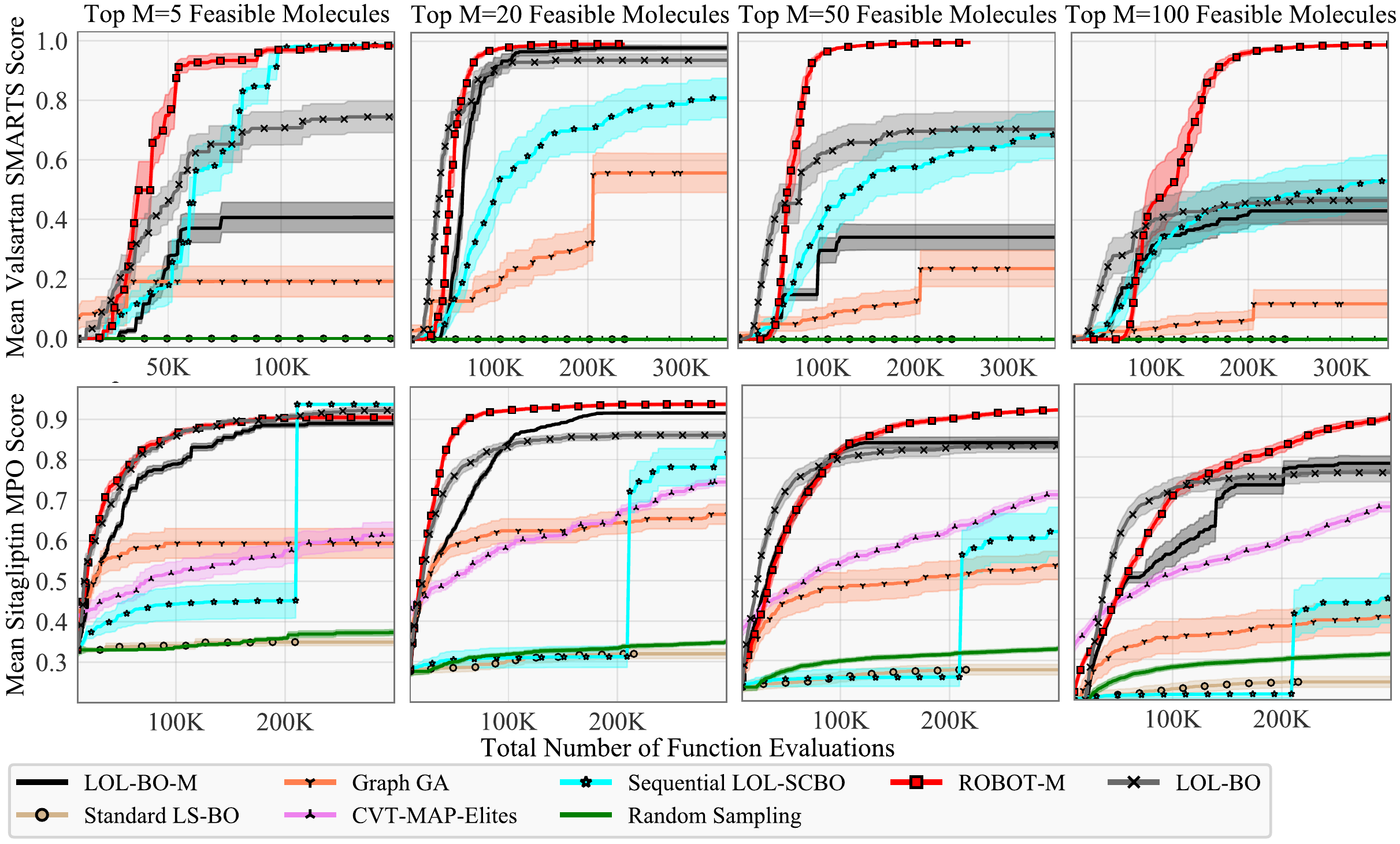}}
        \caption{Comparing to baselines for finding solution sets of $\nopt=5,20,50,100$ solutions to GuacaMol molecule design tasks. Tight constraint ($\tau = -0.4$) used for $\nopt=5$ and relaxed constraint ($\tau = -0.53$) used for $\nopt=20, 50, 100$. Note that in the top row figures, \texttt{Random Sampling}, \texttt{CVT-MAP-ELITES}, and \texttt{Standard LS-BO} all fail to make any progress so the curves overlap..}
        \label{fig:valt053} 
    \end{center}
    \vskip -0.2in
\end{figure*}

\paragraph{Diversity function $\divf{}$ and threshold $\tau$}
Fingerprint similarity (FPS) measures how similar two molecules are \cite{GuacaMol}. 
We therefore define the diversity function $\divf{}$ between two molecules to be the negative of their fingerprint similarity. 
We evaluate finding solutions in two settings: one where we seek a small set of highly diverse solutions, and one where we seek a very large set of moderately diverse solutions. 
Since random pairs of molecules in the ZINC database \cite{irwin2020zinc20} have FPS of $\sim0.4$ on average, we use $\tau = -0.4$ for the highly diverse setting. For the moderately diverse setting, we relax the constraint to $\tau=-0.53$ and sets of up to $100$ molecules.  
\paragraph{Results}
In \autoref{fig:valt053} we depict optimization results for finding $\nopt=5,20,50,100$ diverse solutions on the \valt{} and \siga{} optimization tasks. \ourmethod{} consistently outperforms all other baselines. Notably, across all regimes, Graph-GA produces significantly fewer high scoring yet diverse molecules despite returning an entire population of solutions by default. 

\begin{figure*}[!ht]
    \begin{center}
        \centerline{\includegraphics[width=\textwidth]{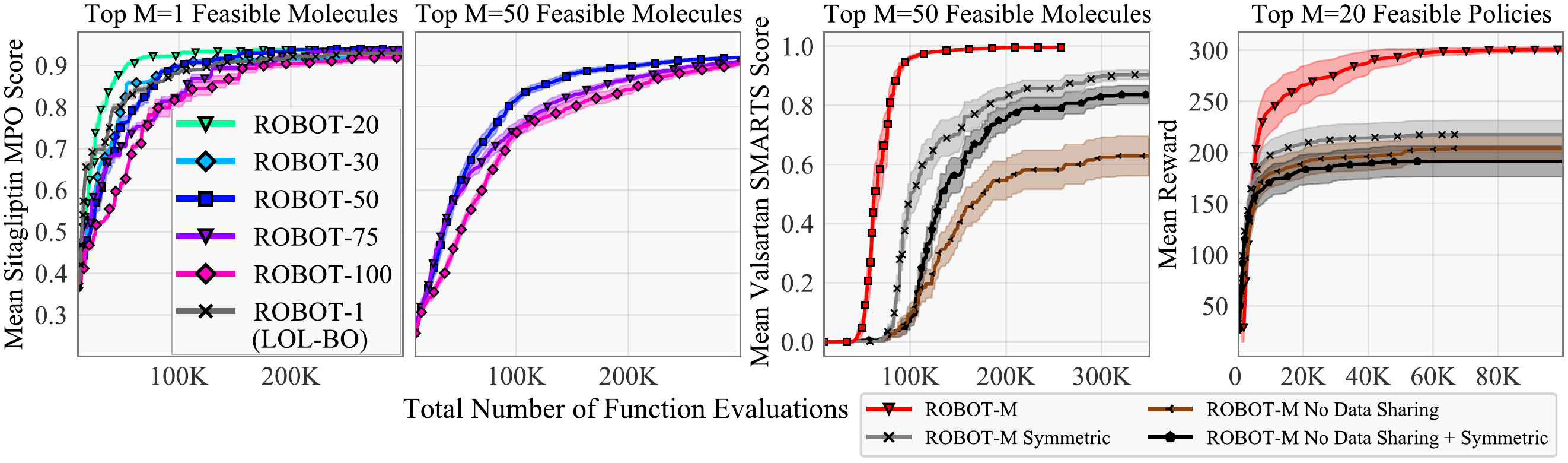}}
        \caption{
            Ablations. \textbf{(Left, Middle Left)} Comparing convergence speed of \ourmethod{}-$K$ on \siga{} with different $K$. \ourmethod{} loses minimal efficiency up to $K=50$, even when searching for $K > M$ molecules. Here we use the relaxed constraint $\tau = -0.53$. See appendix for similar ablation with the tighter constraint $\tau = -0.4$, and repeated ablation on \valt{}.
            \textbf{(Middle Right, Right)} Comparing \ourmethod{} with various components removed for \valt{} and Lunar Lander tasks respectively. }
        \label{fig:valt_ablation} 
    \end{center}
    \vskip -0.2in
\end{figure*}
\subsection{Ablation Studies.} 
In this section, we perform ablation studies evaluating the effect of $K$ on the convergence of \ourmethod{}, as well as the various components of \ourmethod{}.

\subsubsection{How Does $K$ Affect the Convergence Rate of \ourmethod{}-$K$?}
As \ourmethod{}-$K$ is asked to find more diverse solutions (i.e., as $K$ increases), it is reasonable to expect that the convergence in the top $M<K$ solutions becomes slower. 
This is because a fixed evaluation budget must be split $K$ ways for \ourmethod{}-$K$. Thus, when seeking $M$ solutions, the best strategy is to run \ourmethod{}-$M$. However, \ourmethod{} includes a number of features that try to mitigate the loss of efficiency as $K$ increases by sharing information across the $K$ solutions.

To study the overall impact of these features, in this section we run \ourmethod{} to find $\nopt$ solutions, but run \ourmethod{}-$K$ with various $K \geq \nopt$ (e.g., we ask \ourmethod{} to find more than $\nopt$ solutions. We then evaluate the performance of these \ourmethod{}-$K$ runs on \textit{only the top $\nopt$} solutions. If the performance loss as $K$ increases is small, we expect that \ourmethod{}-$K$ should be comparable in convergence to \ourmethod{}-\nopt, even when $K > \nopt$.

Results of this study are in \autoref{fig:valt_ablation} \textbf{Left, Middle Left}, where we ablate \ourmethod{} with various $\nopt$ on \siga{}. Even for $K=50$, $\ourmethod{}$ incurs negligible slowdown compared to finding a single solution ($M=1$), despite the evaluation budget being split between $50$ solutions. In \autoref{fig:valt053ablation} we provide results on \valt{}, and in \autoref{fig:valtsiga04ablation}, we provide an ablation with $\tau = -0.4$, where we observe that \ourmethod{}-$5$ converges at the same speed in terms of finding the single best $1$ value, but is able to find $5$ molecules with nearly the same score with negligible additional evaluations.

\subsubsection{How Do the Different Components of \ourmethod{} Affect Performance?} In \autoref{fig:valt_ablation} \textbf{(Middle Right, Right)}, we evaluate the components of \ourmethod{}. 

To evaluate the gain in optimization performance achieved by using a trust region hierarchy to asymmetrically throw out candidates in acquisition, we compare to the simpler approach of symmetrically discarding all pairs of infeasible candidates. In particular, this involves discarding all pairs of candidates that do not satisfy the diversity constraints from across all trust regions regardless of the trust region hierarchy. 
We refer to this version of \ourmethod{} as \ourmethod{} \texttt{Symmetric}.
Results in \autoref{fig:valt_ablation} \textbf{(Middle Right, Right)} show that \ourmethod{} consistently out preforms \ourmethod{} \texttt{Symmetric}. This indicates that asymmetrically discarding candidates according to the trust region hierarchy is essential for good performance.

To evaluate the gain in optimization performance achieved by allowing collaboration (data sharing) between trust regions, we compare to a version of \ourmethod{} without data sharing. 
In particular, for this version of \ourmethod{} we do not allow trust regions to recenter on data points found other trust regions. Additionally, we use a separate surrogate model for each trust region rather than a single global surrogate model. In this case, each surrogate model is only updated on the data found by its corresponding trust region. 
We refer to this version of \ourmethod{} as \ourmethod{} \texttt{No Data Sharing}.
Results in \autoref{fig:valt_ablation} \textbf{(Middle Right, Right)} show that \ourmethod{} consistently out preforms \ourmethod{} \texttt{No Data Sharing}. This indicates that using a global surrogate model and sharing data across the trust regions is essential for good performance. 

Additionally, we compare to a version of our method where we both do not allow data sharing and symmetrically discard candidates. 
We refer to this version of our method as \ourmethod{} \texttt{No Data Sharing + Symmetric}. 
Results in \autoref{fig:valt_ablation} \textbf{(Middle Right, Right)} show that \ourmethod{} consistently out preforms \ourmethod{} \texttt{No Data Sharing + Symmetric}.


\section{DISCUSSION}
In real world settings, the objective function fed to the optimization routine often tells only part of the story. Practitioners often have preferences for solutions beyond simple objective value. By discovering solutions under \textit{arbitrary user specified} diversity measures rather than input space measures alone, we believe this work may help make Bayesian optimization applicable in scenarios where an ``all-or-nothing'' approach may not be viable. In these scenarios, rather than being a final source of truth, the optimizer can be instead deployed as a tool to generate \textit{suggestions} to a practitioner in a human-in-the-loop fashion, who may ultimately rely on domain expertise to choose the best solution.

\section{Acknowledgements}
\label{sec:acknowledgements}
JRG and NM were supported by NSF award IIS-2145644.

\bibliographystyle{abbrvnat}
\bibliography{ref}

\begin{thebibliography}{54}
\providecommand{\natexlab}[1]{#1}
\providecommand{\url}[1]{\texttt{#1}}
\expandafter\ifx\csname urlstyle\endcsname\relax
  \providecommand{\doi}[1]{doi: #1}\else
  \providecommand{\doi}{doi: \begingroup \urlstyle{rm}\Url}\fi

\bibitem[Balandat et~al.(2020)Balandat, Karrer, Jiang, Daulton, Letham, Wilson,
  and Bakshy]{balandat2020botorch}
M.~Balandat, B.~Karrer, D.~R. Jiang, S.~Daulton, B.~Letham, A.~G. Wilson, and
  E.~Bakshy.
\newblock {BoTorch: A Framework for Efficient Monte-Carlo Bayesian
  Optimization}.
\newblock In \emph{Advances in Neural Information Processing Systems 33}, 2020.

\bibitem[Belakaria et~al.(2019)Belakaria, Deshwal, and Doppa]{mobo2}
S.~Belakaria, A.~Deshwal, and J.~R. Doppa.
\newblock Max-value entropy search for multi-objective bayesian optimization.
\newblock In H.~Wallach, H.~Larochelle, A.~Beygelzimer, F.~d\textquotesingle
  Alch\'{e}-Buc, E.~Fox, and R.~Garnett, editors, \emph{Advances in Neural
  Information Processing Systems}, volume~32. Curran Associates, Inc., 2019.
\newblock URL
  \url{https://proceedings.neurips.cc/paper/2019/file/82edc5c9e21035674d481640448049f3-Paper.pdf}.

\bibitem[Brown et~al.(2019)Brown, Fiscato, Segler, and Vaucher]{GuacaMol}
N.~Brown, M.~Fiscato, M.~H. Segler, and A.~C. Vaucher.
\newblock Guacamol: Benchmarking models for de novo molecular design.
\newblock \emph{Journal of Chemical Information and Modeling}, 59\penalty0
  (3):\penalty0 1096–1108, Mar 2019.

\bibitem[Daulton et~al.(2021)Daulton, Eriksson, Balandat, and Bakshy]{morbo}
S.~Daulton, D.~Eriksson, M.~Balandat, and E.~Bakshy.
\newblock Multi-objective bayesian optimization over high-dimensional search
  spaces.
\newblock \emph{arXiv preprint arXiv:2109.10964}, 2021.

\bibitem[Eissman et~al.(2018)Eissman, Levy, Shu, Bartzsch, and
  Ermon]{eissman2018bayesian}
S.~Eissman, D.~Levy, R.~Shu, S.~Bartzsch, and S.~Ermon.
\newblock Bayesian optimization and attribute adjustment.
\newblock In \emph{Proc. 34th Conference on Uncertainty in Artificial
  Intelligence}, 2018.

\bibitem[Eriksson and Poloczek(2021)]{scbo}
D.~Eriksson and M.~Poloczek.
\newblock Scalable constrained {Bayesian} optimization.
\newblock In \emph{International Conference on Artificial Intelligence and
  Statistics}, pages 730--738. PMLR, 2021.

\bibitem[Eriksson et~al.(2019)Eriksson, Pearce, Gardner, Turner, and
  Poloczek]{TuRBO}
D.~Eriksson, M.~Pearce, J.~Gardner, R.~D. Turner, and M.~Poloczek.
\newblock Scalable global optimization via local {Bayesian} optimization.
\newblock In \emph{Advances in Neural Information Processing Systems},
  volume~32. Curran Associates, Inc., 2019.

\bibitem[Frazier(2018)]{frazier2018tutorial}
P.~I. Frazier.
\newblock A tutorial on {Bayesian} optimization.
\newblock \emph{arXiv preprint arXiv:1807.02811}, 2018.

\bibitem[Gaier et~al.(2018)Gaier, Asteroth, and Mouret]{gaier2018data}
A.~Gaier, A.~Asteroth, and J.-B. Mouret.
\newblock Data-efficient design exploration through surrogate-assisted
  illumination.
\newblock \emph{Evolutionary computation}, 26\penalty0 (3):\penalty0 381--410,
  2018.

\bibitem[Gardner et~al.(2014)Gardner, Kusner, Xu, Weinberger, and
  Cunningham]{cei}
J.~R. Gardner, M.~J. Kusner, Z.~E. Xu, K.~Q. Weinberger, and J.~P. Cunningham.
\newblock Bayesian optimization with inequality constraints.
\newblock In \emph{ICML}, pages 937--945, 2014.

\bibitem[Gardner et~al.(2018)Gardner, Pleiss, Bindel, Weinberger, and
  Wilson]{gardner2018gpytorch}
J.~R. Gardner, G.~Pleiss, D.~Bindel, K.~Q. Weinberger, and A.~G. Wilson.
\newblock Gpytorch: Blackbox matrix-matrix {Gaussian} process inference with
  gpu acceleration.
\newblock \emph{arXiv preprint arXiv:1809.11165}, 2018.

\bibitem[G{\'o}mez-Bombarelli et~al.(2018)G{\'o}mez-Bombarelli, Wei, Duvenaud,
  Hern{\'a}ndez-Lobato, S{\'a}nchez-Lengeling, Sheberla, Aguilera-Iparraguirre,
  Hirzel, Adams, and Aspuru-Guzik]{gomez2018automatic}
R.~G{\'o}mez-Bombarelli, J.~N. Wei, D.~Duvenaud, J.~M. Hern{\'a}ndez-Lobato,
  B.~S{\'a}nchez-Lengeling, D.~Sheberla, J.~Aguilera-Iparraguirre, T.~D.
  Hirzel, R.~P. Adams, and A.~Aspuru-Guzik.
\newblock Automatic chemical design using a data-driven continuous
  representation of molecules.
\newblock \emph{ACS central science}, 4\penalty0 (2):\penalty0 268--276, 2018.

\bibitem[Graff et~al.(2021)Graff, Shakhnovich, and Coley]{Graff2021-wr}
D.~E. Graff, E.~I. Shakhnovich, and C.~W. Coley.
\newblock Accelerating high-throughput virtual screening through molecular
  pool-based active learning.
\newblock \emph{Chemical science}, 12\penalty0 (22):\penalty0 7866--7881, Apr.
  2021.

\bibitem[Grosnit et~al.(2021)Grosnit, Tutunov, Maraval, Griffiths,
  Cowen{-}Rivers, Yang, Zhu, Lyu, Chen, Wang, Peters, and Bou{-}Ammar]{Huawei}
A.~Grosnit, R.~Tutunov, A.~M. Maraval, R.~Griffiths, A.~I. Cowen{-}Rivers,
  L.~Yang, L.~Zhu, W.~Lyu, Z.~Chen, J.~Wang, J.~Peters, and H.~Bou{-}Ammar.
\newblock High-dimensional {Bayesian} optimisation with variational
  autoencoders and deep metric learning.
\newblock \emph{CoRR}, abs/2106.03609, 2021.

\bibitem[Hansen(2016)]{CMAES}
N.~Hansen.
\newblock The cma evolution strategy: A tutorial, 2016.
\newblock URL \url{https://arxiv.org/abs/1604.00772}.

\bibitem[Hensman et~al.(2013)Hensman, Fusi, and Lawrence]{hensman2013gaussian}
J.~Hensman, N.~Fusi, and N.~D. Lawrence.
\newblock Gaussian processes for big data.
\newblock \emph{Proceedings of the Twenty-Ninth Conference on Uncertainty in
  Artificial Intelligence}, 2013.

\bibitem[Hern{\'a}ndez-Lobato et~al.(2016)Hern{\'a}ndez-Lobato, Gelbart, Adams,
  Hoffman, and Ghahramani]{pesc}
J.~M. Hern{\'a}ndez-Lobato, M.~A. Gelbart, R.~P. Adams, M.~W. Hoffman, and
  Z.~Ghahramani.
\newblock A general framework for constrained {Bayesian} optimization using
  information-based search.
\newblock \emph{The Journal of Machine Learning Research}, 17\penalty0
  (1):\penalty0 5549--5601, 2016.
\newblock Code available at: \url{https://github.com/HIPS/Spearmint/tree/PESC}.
  Last accessed on 02/03/2020.

\bibitem[Hern{\'a}ndez-Lobato et~al.(2017)Hern{\'a}ndez-Lobato, Requeima,
  Pyzer-Knapp, and Aspuru-Guzik]{hernandez2017parallel}
J.~M. Hern{\'a}ndez-Lobato, J.~Requeima, E.~O. Pyzer-Knapp, and
  A.~Aspuru-Guzik.
\newblock Parallel and distributed {Thompson} sampling for large-scale
  accelerated exploration of chemical space.
\newblock In D.~Precup and Y.~W. Teh, editors, \emph{Proceedings of the 34th
  International Conference on Machine Learning}, volume~70, pages 1470--1479.
  PMLR, 2017.

\bibitem[Hernández-Lobato et~al.(2015)Hernández-Lobato, Hernández-Lobato,
  Shah, and Adams]{mobo1}
D.~Hernández-Lobato, J.~M. Hernández-Lobato, A.~Shah, and R.~P. Adams.
\newblock Predictive entropy search for multi-objective bayesian optimization,
  2015.
\newblock URL \url{https://arxiv.org/abs/1511.05467}.

\bibitem[Irwin et~al.(2020)Irwin, Tang, Young, Dandarchuluun, Wong,
  Khurelbaatar, Moroz, Mayfield, and Sayle]{irwin2020zinc20}
J.~J. Irwin, K.~G. Tang, J.~Young, C.~Dandarchuluun, B.~R. Wong,
  M.~Khurelbaatar, Y.~S. Moroz, J.~Mayfield, and R.~A. Sayle.
\newblock Zinc20—a free ultralarge-scale chemical database for ligand
  discovery.
\newblock \emph{Journal of chemical information and modeling}, 60\penalty0
  (12):\penalty0 6065--6073, 2020.

\bibitem[Jankowiak et~al.(2020)Jankowiak, Pleiss, and Gardner]{PPGPR}
M.~Jankowiak, G.~Pleiss, and J.~R. Gardner.
\newblock Parametric gaussian process regressors.
\newblock In \emph{Proceedings of the 37th International Conference on Machine
  Learning}, ICML'20. JMLR.org, 2020.

\bibitem[Jensen(2019)]{graphGA}
J.~Jensen.
\newblock A graph-based genetic algorithm and generative model/{Monte Carlo}
  tree search for the exploration of chemical space. chem sci 10 (12):
  3567--3572, 2019.

\bibitem[Jin et~al.(2018)Jin, Barzilay, and Jaakkola]{JTVAE}
W.~Jin, R.~Barzilay, and T.~S. Jaakkola.
\newblock Junction tree variational autoencoder for molecular graph generation.
\newblock In \emph{International Conference on Machine Learning}. PMLR, 2018.

\bibitem[Jones et~al.(1998)Jones, Schonlau, and Welch]{jones1998efficient}
D.~R. Jones, M.~Schonlau, and W.~J. Welch.
\newblock Efficient global optimization of expensive black-box functions.
\newblock \emph{Journal of Global optimization}, 13\penalty0 (4):\penalty0
  455--492, 1998.

\bibitem[Kandasamy et~al.(2015)Kandasamy, Schneider, and
  P{\'o}czos]{kandasamy2015high}
K.~Kandasamy, J.~Schneider, and B.~P{\'o}czos.
\newblock High dimensional {Bayesian} optimisation and bandits via additive
  models.
\newblock In \emph{International conference on machine learning}, pages
  295--304. PMLR, 2015.

\bibitem[Kent and Branke(2020)]{BOP-ELITES}
P.~Kent and J.~Branke.
\newblock Bop-elites, a bayesian optimisation algorithm for quality-diversity
  search, 2020.
\newblock URL \url{https://arxiv.org/abs/2005.04320}.

\bibitem[Konakovic~Lukovic et~al.(2020)Konakovic~Lukovic, Tian, and
  Matusik]{R5}
M.~Konakovic~Lukovic, Y.~Tian, and W.~Matusik.
\newblock Diversity-guided multi-objective bayesian optimization with batch
  evaluations.
\newblock In H.~Larochelle, M.~Ranzato, R.~Hadsell, M.~Balcan, and H.~Lin,
  editors, \emph{Advances in Neural Information Processing Systems}, volume~33,
  pages 17708--17720. Curran Associates, Inc., 2020.
\newblock URL
  \url{https://proceedings.neurips.cc/paper/2020/file/cd3109c63bf4323e6b987a5923becb96-Paper.pdf}.

\bibitem[Letham et~al.(2019)Letham, Karrer, Ottoni, and
  Bakshy]{letham2019noisyei}
B.~Letham, B.~Karrer, G.~Ottoni, and E.~Bakshy.
\newblock Constrained {B}ayesian optimization with noisy experiments.
\newblock \emph{{Bayesian} Analysis}, 14\penalty0 (2):\penalty0 495--519, 2019.

\bibitem[Letham et~al.(2020)Letham, Calandra, Rai, and Bakshy]{letham2019re}
B.~Letham, R.~Calandra, A.~Rai, and E.~Bakshy.
\newblock Re-examining linear embeddings for high-dimensional {Bayesian}
  optimization.
\newblock \emph{Advances in neural information processing systems},
  33:\penalty0 1546--1558, 2020.

\bibitem[Maus et~al.(2022)Maus, Jones, Moore, Kusner, Bradshaw, and
  Gardner]{maus2022local}
N.~Maus, H.~T. Jones, J.~S. Moore, M.~J. Kusner, J.~Bradshaw, and J.~R.
  Gardner.
\newblock Local latent space bayesian optimization over structured inputs.
\newblock \emph{arXiv preprint arXiv:2201.11872}, 2022.

\bibitem[Mo{\v{c}}kus(1975)]{movckus1975bayesian}
J.~Mo{\v{c}}kus.
\newblock On bayesian methods for seeking the extremum.
\newblock In \emph{Optimization Techniques IFIP Technical Conference:
  Novosibirsk, July 1--7, 1974}, pages 400--404. Springer, 1975.

\bibitem[Mouret and Clune(2015)]{MAP-Elites}
J.-B. Mouret and J.~Clune.
\newblock Illuminating search spaces by mapping elites, 2015.

\bibitem[Mouret and Maguire(2020)]{Multitask-MAP-ELITES}
J.-B. Mouret and G.~Maguire.
\newblock Quality diversity for multi-task optimization.
\newblock In \emph{Proceedings of the 2020 Genetic and Evolutionary Computation
  Conference}. {ACM}, jun 2020.

\bibitem[Mutny and Krause(2018)]{mutny2019efficient}
M.~Mutny and A.~Krause.
\newblock Efficient high dimensional {Bayesian} optimization with additivity
  and quadrature {Fourier} features.
\newblock In \emph{Advances in Neural Information Processing Systems 31}, pages
  9019--9030, 2018.

\bibitem[Negoescu et~al.(2011)Negoescu, Frazier, and
  Powell]{negoescu2011knowledge}
D.~M. Negoescu, P.~I. Frazier, and W.~B. Powell.
\newblock The knowledge-gradient algorithm for sequencing experiments in drug
  discovery.
\newblock \emph{INFORMS Journal on Computing}, 23\penalty0 (3):\penalty0
  346--363, 2011.

\bibitem[Neiswanger et~al.(2021)Neiswanger, Wang, and Ermon]{R3-2}
W.~Neiswanger, K.~A. Wang, and S.~Ermon.
\newblock Bayesian algorithm execution: Estimating computable properties of
  black-box functions using mutual information.
\newblock In M.~Meila and T.~Zhang, editors, \emph{Proceedings of the 38th
  International Conference on Machine Learning}, volume 139 of
  \emph{Proceedings of Machine Learning Research}, pages 8005--8015. PMLR,
  18--24 Jul 2021.
\newblock URL \url{https://proceedings.mlr.press/v139/neiswanger21a.html}.

\bibitem[Rasmussen(2003)]{rasmussen2003gaussian}
C.~E. Rasmussen.
\newblock Gaussian processes in machine learning.
\newblock In \emph{Summer School on Machine Learning}, pages 63--71. Springer,
  2003.

\bibitem[Regis and Shoemaker(2007)]{regis2007stochastic}
R.~G. Regis and C.~A. Shoemaker.
\newblock A stochastic radial basis function method for the global optimization
  of expensive functions.
\newblock \emph{INFORMS Journal on Computing}, 19\penalty0 (4):\penalty0
  497--509, 2007.

\bibitem[Shahriari et~al.(2015)Shahriari, Swersky, Wang, Adams, and
  De~Freitas]{shahriari2015taking}
B.~Shahriari, K.~Swersky, Z.~Wang, R.~P. Adams, and N.~De~Freitas.
\newblock Taking the human out of the loop: {A} review of {Bayesian}
  optimization.
\newblock \emph{Proceedings of the IEEE}, 104\penalty0 (1):\penalty0 148--175,
  2015.

\bibitem[Siivola et~al.(2021)Siivola, Paleyes, Gonz{\'a}lez, and
  Vehtari]{siivola2021good}
E.~Siivola, A.~Paleyes, J.~Gonz{\'a}lez, and A.~Vehtari.
\newblock Good practices for {Bayesian} optimization of high dimensional
  structured spaces.
\newblock \emph{Applied AI Letters}, 2\penalty0 (2):\penalty0 e24, 2021.

\bibitem[Snoek et~al.(2012{\natexlab{a}})Snoek, Larochelle, and Adams]{SnoekBO}
J.~Snoek, H.~Larochelle, and R.~P. Adams.
\newblock Practical bayesian optimization of machine learning algorithms,
  2012{\natexlab{a}}.
\newblock URL \url{https://arxiv.org/abs/1206.2944}.

\bibitem[Snoek et~al.(2012{\natexlab{b}})Snoek, Larochelle, and
  Adams]{snoek2012practical}
J.~Snoek, H.~Larochelle, and R.~P. Adams.
\newblock Practical {Bayesian} optimization of machine learning algorithms.
\newblock \emph{Advances in neural information processing systems}, 25,
  2012{\natexlab{b}}.

\bibitem[Stanton et~al.(2022)Stanton, Maddox, Gruver, Maffettone, Delaney,
  Greenside, and Wilson]{lambo}
S.~Stanton, W.~Maddox, N.~Gruver, P.~Maffettone, E.~Delaney, P.~Greenside, and
  A.~G. Wilson.
\newblock Accelerating bayesian optimization for biological sequence design
  with denoising autoencoders, 2022.

\bibitem[Su et~al.(2020)Su, Liu, Zheng, Zhou, and Zheng]{owd}
H.~Su, S.~Liu, B.~Zheng, X.~Zhou, and K.~Zheng.
\newblock A survey of trajectory distance measures and performance evaluation.
\newblock \emph{The VLDB Journal}, 29:\penalty0 13–32, 2020.

\bibitem[Titsias(2009)]{titsias2009variational}
M.~Titsias.
\newblock Variational learning of inducing variables in sparse gaussian
  processes.
\newblock In \emph{Artificial intelligence and statistics}, pages 567--574.
  PMLR, 2009.

\bibitem[Tripp et~al.(2020)Tripp, Daxberger, and
  Hern{\'{a}}ndez{-}Lobato]{Weighted_Retraining}
A.~Tripp, E.~A. Daxberger, and J.~M. Hern{\'{a}}ndez{-}Lobato.
\newblock Sample-efficient optimization in the latent space of deep generative
  models via weighted retraining.
\newblock In \emph{Advances in Neural Information Processing Systems 33}, 2020.

\bibitem[Turchetta et~al.(2019)Turchetta, Krause, and Trimpe]{mobo3}
M.~Turchetta, A.~Krause, and S.~Trimpe.
\newblock Robust model-free reinforcement learning with multi-objective
  bayesian optimization, 2019.
\newblock URL \url{https://arxiv.org/abs/1910.13399}.

\bibitem[Turner et~al.(2021)Turner, Eriksson, McCourt, Kiili, Laaksonen, Xu,
  and Guyon]{turner2021bayesian}
R.~Turner, D.~Eriksson, M.~McCourt, J.~Kiili, E.~Laaksonen, Z.~Xu, and
  I.~Guyon.
\newblock {Bayesian} optimization is superior to random search for machine
  learning hyperparameter tuning: Analysis of the black-box optimization
  challenge 2020.
\newblock In \emph{NeurIPS 2020 Competition and Demonstration Track}, pages
  3--26, 2021.

\bibitem[Vassiliades et~al.(2016)Vassiliades, Chatzilygeroudis, and
  Mouret]{CV-MAP-Elites}
V.~Vassiliades, K.~Chatzilygeroudis, and J.-B. Mouret.
\newblock Using centroidal voronoi tessellations to scale up the
  multi-dimensional archive of phenotypic elites algorithm, 2016.

\bibitem[Wang et~al.(2016)Wang, Hutter, Zoghi, Matheson, and
  De~Freitas]{wang2016rembo}
Z.~Wang, F.~Hutter, M.~Zoghi, D.~Matheson, and N.~De~Freitas.
\newblock Bayesian optimization in a billion dimensions via random embeddings.
\newblock \emph{J. Artif. Int. Res.}, 55\penalty0 (1):\penalty0 361--387, Jan.
  2016.

\bibitem[Wang et~al.(2018)Wang, Gehring, Kohli, and Jegelka]{ebo}
Z.~Wang, C.~Gehring, P.~Kohli, and S.~Jegelka.
\newblock Batched large-scale {Bayesian} optimization in high-dimensional
  spaces.
\newblock In \emph{International Conference on Artificial Intelligence and
  Statistics}, volume~84, pages 745--754, 2018.

\bibitem[Wessing and Preuss(2017)]{R2}
S.~Wessing and M.~Preuss.
\newblock The true destination of ego is multi-local optimization, 2017.
\newblock URL \url{https://arxiv.org/abs/1704.05724}.

\bibitem[Williams et~al.(2015)Williams, Bilsland, Sparkes, Aubrey, Young,
  Soldatova, De~Grave, Ramon, de~Clare, Sirawaraporn, Oliver, and
  King]{Williams2015-cp}
K.~Williams, E.~Bilsland, A.~Sparkes, W.~Aubrey, M.~Young, L.~N. Soldatova,
  K.~De~Grave, J.~Ramon, M.~de~Clare, W.~Sirawaraporn, S.~G. Oliver, and R.~D.
  King.
\newblock Cheaper faster drug development validated by the repositioning of
  drugs against neglected tropical diseases.
\newblock \emph{Journal of the Royal Society, Interface / the Royal Society},
  12\penalty0 (104):\penalty0 20141289, Mar. 2015.

\bibitem[Xie et~al.(2021)Xie, Xu, Ma, and Mei]{R3-1}
Y.~Xie, Z.~Xu, J.~Ma, and Q.~Mei.
\newblock How much space has been explored? measuring the chemical space
  covered by databases and machine-generated molecules, 2021.
\newblock URL \url{https://arxiv.org/abs/2112.12542}.

\end{thebibliography}


\clearpage
\onecolumn

\hsize\textwidth\linewidth\hsize\toptitlebar 
{\centering{\Large\bfseries Discovering Many Diverse Solutions with Bayesian Optimization \\ Supplementary Materials \par}}
\bottomtitlebar

\appendix

\section{APPENDIX}

\subsection{\ourmethod{} Initialization Details}

\ourmethod{} is initialized with a small set of quasi-random data. 
The $M$ trust regions are initialized one at a time (in rank-order). 
Each trust region is centered on the highest-scoring point in the initial data which is sufficiently diverse from the centers of all higher-ranking trust regions. 
See \autoref{tab:init} for the number of initialization data points used for each task.
All baseline methods use the same number of initialization data points for each task. 
\begin{table}[h]
\caption{Number of random initialization data points used to initialize all methods for each task in \autoref{sec:experiments} } \label{tab:init}
\begin{center}
\begin{tabular}{ll}
\textbf{TASK}  &\textbf{NUMBER OF POINTS} \\
\hline \\
Rover   & $1024$ \\
Lunar Lander  & $1024$ \\
Stock Portfolio Diversification & $1024$ \\
Molecule Tasks  & $10$,$000$ \\
\end{tabular}
\end{center}
\end{table}

\subsection{Additional Implementation Details}
\label{sec:detials}
In this section, we provide additional implementation details for \ourmethod{} and baseline methods.

\subsubsection{Surrogate Model}
As discussed in \autoref{sec:experiments}, we use a PPGPR \cite{PPGPR} surrogate model. To maintain fair comparison, we use the same surrogate model with the same configuration for \ourmethod{} and all baseline Bayesian optimization methods. 
We use a PPGPR with a constant mean and standard RBF kernel. 
Due to the high dimensionality of our chosen tasks, use a deep kernel (several fully connected layers between the search space and the GP kernel). 
In particular, we use two fully connected layers with $32$ nodes each. 
We update the parameters of the PPGPR during optimization by training it on collected data using the Adam optimizer with a learning rate of $0.001$. 
The PPGPR is initially trained on a small set of initialization data for $20$ epochs.
For the number of initialization data points used for each task, see \autoref{tab:init}.
On each step of optimization, the model is updated on the newly collected data for $2$ epochs. 
This is kept consistent across all Bayesian optimization methods. 

\subsubsection{Details for Non-BO Methods}

For the CVT-Map-Elits method, we use all default hyper-parameters from \cite{CV-MAP-Elites} with the ``Number of Niches" parameter set equal to $M$ (the number of diverse solutions we want to find).
For the Graph GA method we use the implementation provided by \cite{GuacaMol} which returns a population of molecules by default. We use the default hyper-parameters from \cite{graphGA}.  
For random sampling, we select query points in the search space uniformly at random, and check to make sure that we never select the same point twice. 

\subsection{Additional Experiments}
In this section, we provide some additional experimental results, including analysis for \ourmethod{} on the two additional optimization tasks mentioned in \autoref{sec:experiments}. 

\subsubsection{Stock Portfolio Diversification}
\label{sec:stocks} 
In portfolio optimization, the goal is to find a $500$-dimensional weight vector $\bx$ giving the optimal fraction of a principal that should be invested in each company in the S$\&$P 500 in order to maximize return while minimizing volatility. We quantify the return/volatility trade-off by directly maximizing the Sharpe ratio of the portfolio, which we compute using the past three years of data from the S$\&$P $500$. For our toy setting, we make the simplifying assumption that risk-free return is always zero, such that the Sharpe ratio is defined by the equation: $\frac{ROI}{\sigma \sqrt{252}}$. Here, the $ROI$ is the total return on the investment, $\sigma$ is the total volatility (the standard deviation of day-to-day changes in return) and $\sqrt{252}$ is a standard normalizing constant which accounts for the number of trading days in a year. 

\paragraph{Diversity function $\divf{}$ and threshold $\tau$}
We define the diversity function $\divf{}$ between two portfolios to be the maximum integer $k$ such that the two portfolios have no stocks in common in their top $k$ invested stocks. We require feasible portfolios to have $\tau = 10$ different companies in the top $10$ -- that is, the top 10 stocks must be disjoint.

\paragraph{Results.} 
\begin{figure*}[!ht]
    \vskip 0.2in
    \begin{center}
        \centerline{\includegraphics[width=0.8\textwidth]{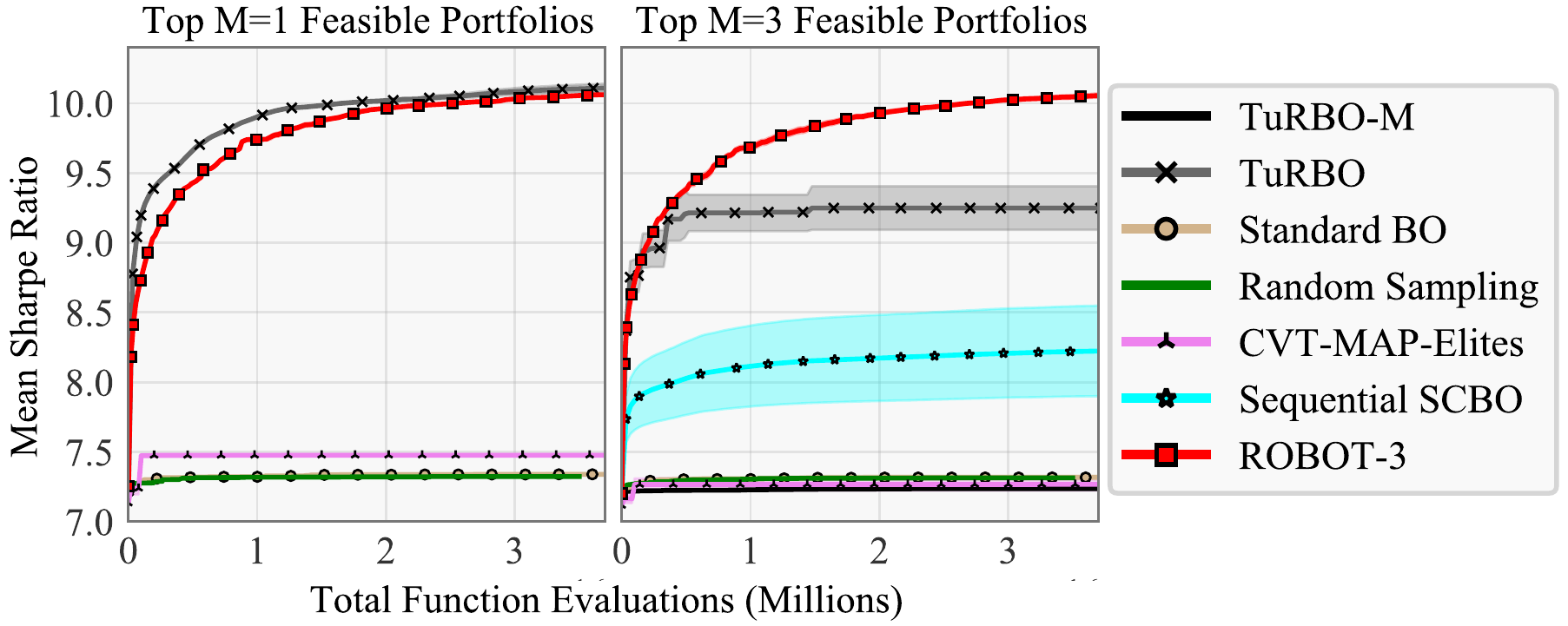}}
        \caption{S$\&$P 500 optimization task. Feasible portfolios have a minimum diversity of $\tau = 10$. }
        \label{fig:stocks}
    \end{center}
    \vskip -0.2in
\end{figure*}
Results on this task for $\nopt = 1$ and $\nopt = 3$ trajectories are displayed in \autoref{fig:stocks}. \ourmethod{}-$3$ converges at nearly the same speed in terms of finding the single best 1 portfolio, and is simultaneously able to find $3$ portfolios with nearly the same Sharpe ratio with negligible additional evaluations. 

\subsubsection{Penalized Log P Optimization}
\label{sec:logp} 
In this section, we provide results for a third molecular optimization task - \logp{}. As in \autoref{sec:mols}, we use negative fingerprint similarity for our diversity function $\divf{}$ and consider the two settings with tight and relaxed constraints ($\tau = -0.4$ and $\tau = -0.53$ respectively). Since \citet{maus2022local} showed that the \logp{} oracle can be exploited to produce high-scoring molecules which are wholly unrealistic, we constrain the decoder of the VAE to producing SELFIES strings of 400 tokens or fewer. With this added constraint, \texttt{LOL-BO} achieves \logp{} scores of $\sim100$ rather than $\sim500$ \autoref{fig:logp04}. Since the \logp{} scores reported by \citet{graphGA} are not competitive with $\sim100$ (\citet{graphGA} reports maximum \logp{} scores of $\sim12$), we do not compare to Graph-GA for this task. 

\paragraph{Results with smaller $\nopt$ and tighter constraints.}
In \autoref{fig:logp04} we depict optimization results for finding $\nopt=1$ and $\nopt=5$ diverse solutions on the \logp{} optimization task. In the $\nopt=1$ case (left panel in \autoref{fig:logp04}), \ourmethod{}-$3$, \ourmethod{}-$5$, and \ourmethod{}-$10$ surprisingly appear to find the best molecule a bit faster than \texttt{LOL-BO} despite simultaneously searching for other diverse solutions. A possible explanation for this is that \texttt{LOL-BO} uses only one trust region and is therefore more myopic than \ourmethodM{}. It appears that when $M$ is small, the additional exploration done by \ourmethodM{} can actually allow for a single best solution to be found more quickly. Furthermore, despite searching for $5$ diverse molecules, \ourmethod{}-$5$ and is simultaneously able to find $5$ molecules with nearly the same score with negligible additional evaluations (right panel in \autoref{fig:logp04}). 
\begin{figure*}[!ht]
    \vskip 0.2in
    \begin{center}
        \centerline{\includegraphics[width=0.8\textwidth]{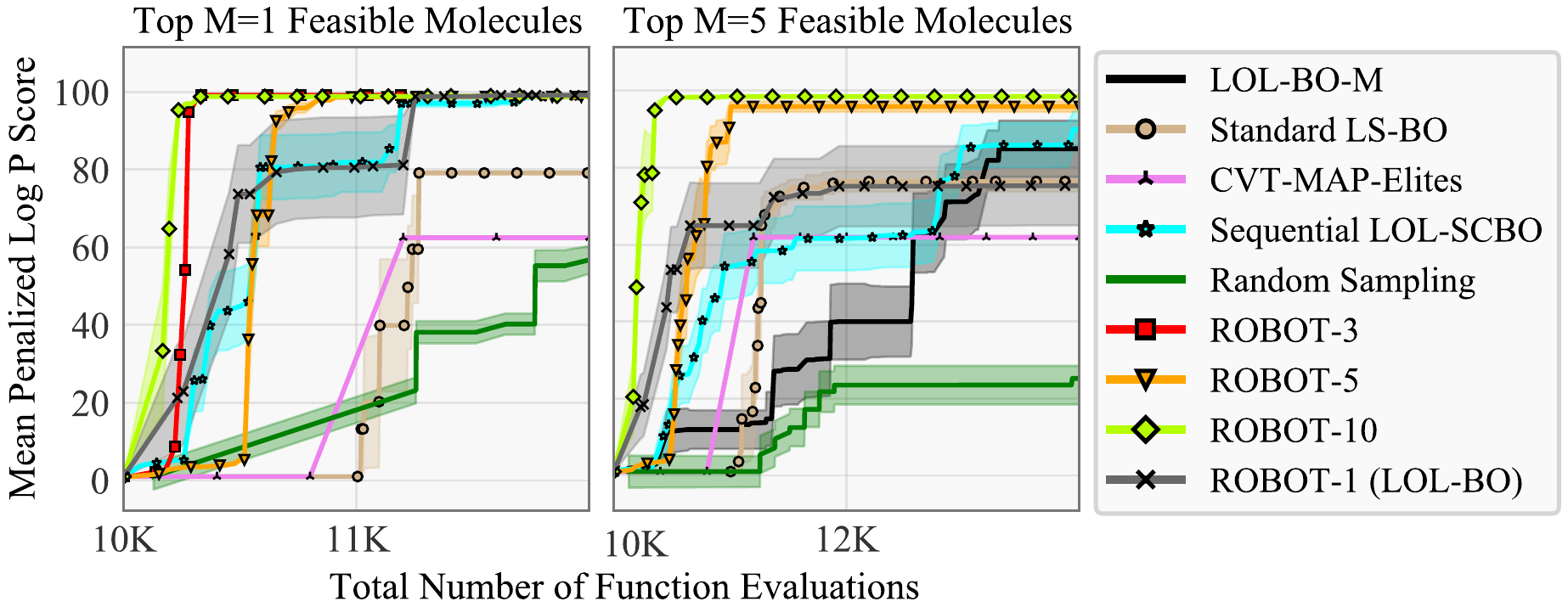}}
        \caption{\logp{} molecular optimization task. Feasible molecules have a maximum fingerprint similarity of 0.4.    }
        \label{fig:logp04} 
    \end{center}
    \vskip -0.2in
\end{figure*}
\paragraph{Results with larger $\nopt$ and relaxed constraints.}
In \autoref{fig:logp53} we depict optimization results for finding larger sets of $\nopt=1, 20$, and $50$ diverse molecules. 
Even up to asking $\ourmethod{}$ to find $50$ solutions, we incur very little slowdown compared to finding a single solution. Furthermore, \ourmethod{} is able to find a full set of 20 and 50 high-scoring molecules with a small number of additional evaluations. 
\begin{figure*}[!ht]
    \vskip 0.2in
    \begin{center}
        \centerline{\includegraphics[width=0.8\textwidth]{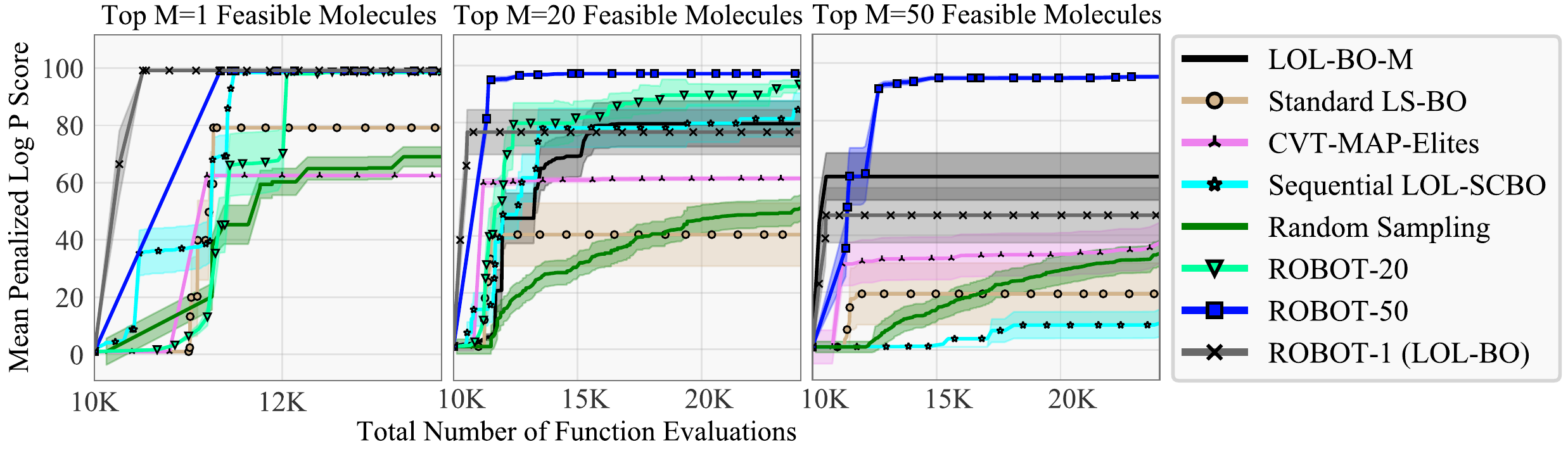}}
        \caption{\logp{} molecular optimization task. Feasible molecules have a maximum fingerprint similarity of 0.53.    }
        \label{fig:logp53} 
    \end{center}
    \vskip -0.2in
\end{figure*}

\subsubsection{Fingerprint Similarity Thresholds } 
\label{sec:fps} 
In \autoref{fig:zinc fps hist}, we provide visualization of the distribution of fingerprint similarities between molecules in the Zinc20 data-set. This distribution informed our choices $\tau = -0.4$ and $\tau = -0.53$ for molecular optimization tasks. 
\begin{figure*}[!ht]
    \begin{center}
        \centerline{\includegraphics[width=0.45\textwidth]{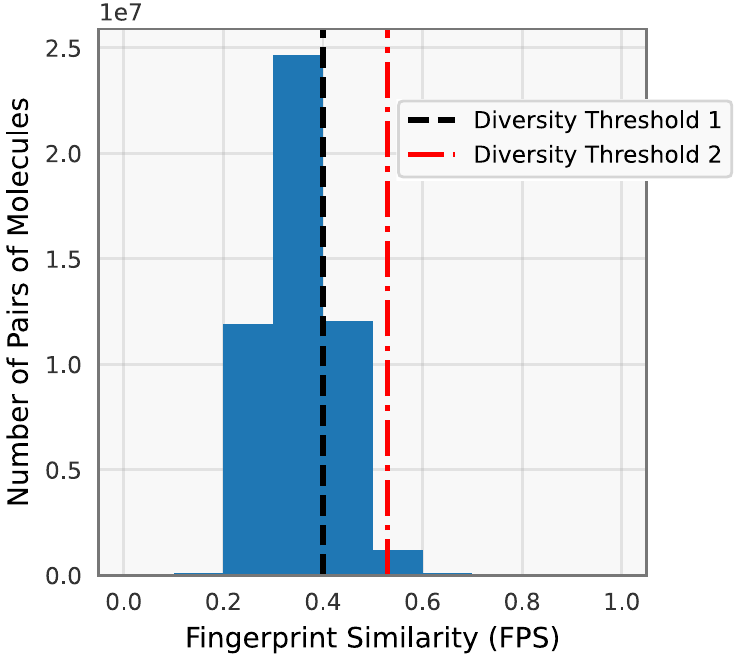}}
        \caption{ Fingerprint similarities between all pairs of $10$,$000$ randomly selected molecules from the Zinc20 data-set. Our chosen diversity constraint thresholds of $-\tau = 0.4$ (threshold 1) $-\tau= 0.53$ (threshold 2) are shown with vertical lines. }
        \label{fig:zinc fps hist} 
    \end{center}
\end{figure*}

\subsubsection{Additional Ablations} 
\label{sec:lunar multiple trs} 
In \autoref{fig:individual-trs} we plot the best objective value found by individual rank-ordered trust regions during optimization runs of \ourmethod{}-$M$ on \valt{} and Lunar Lander tasks. We observe that higher-ranking trust regions are able to converge more quickly since they are unimpeded by lower-ranking trust regions. 
\begin{figure*}[!ht]
    \begin{center}
        \centerline{\includegraphics[width=0.8\textwidth]{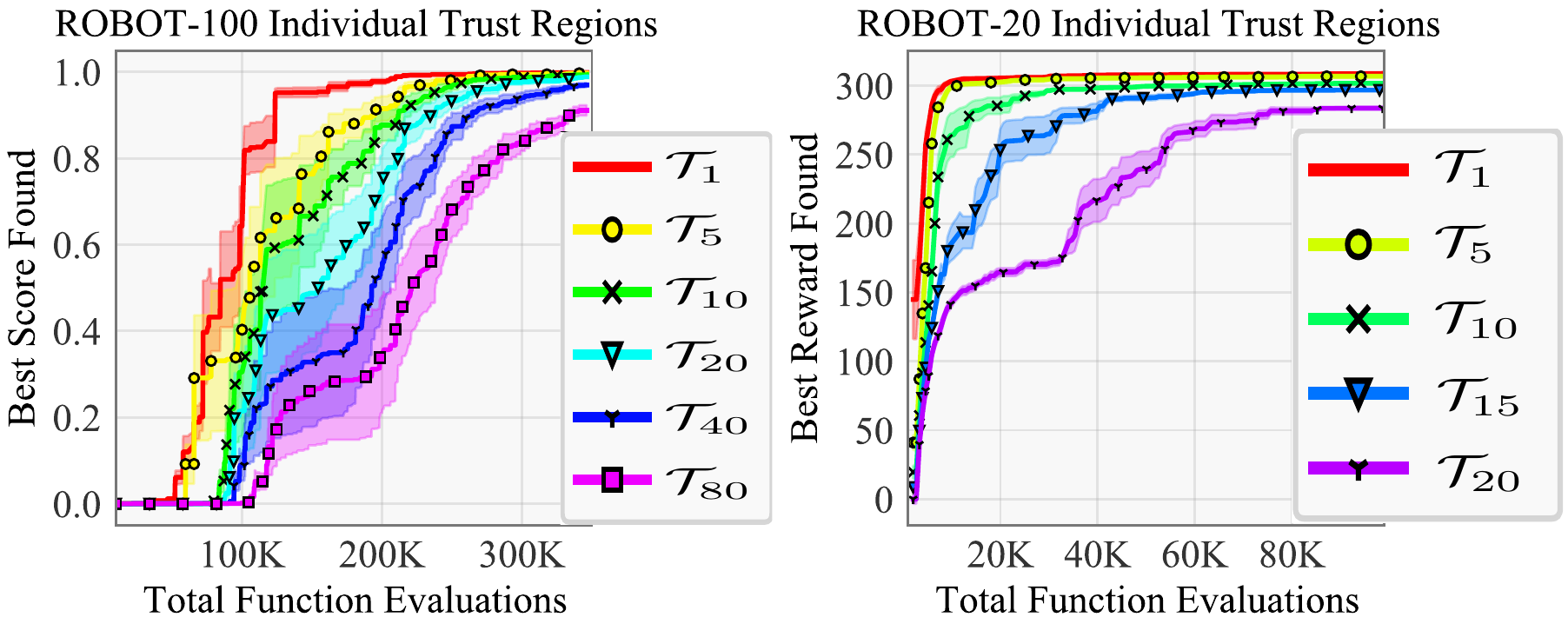}}
        \caption{ Objective found by individual trust regions during optimization with  \ourmethod{}. Higher-rank trust regions converge faster than lower-ranking trust regions. \textbf{(Left)} Optimization of \valt{} with \ourmethod{}-100, $\tau = -0.53$. \textbf{(Right)} Optimization of Lunar Lander with \ourmethod{}-20, $\tau = 0.6$.  }
        \label{fig:individual-trs} 
    \end{center}
\end{figure*}

In \autoref{fig:valt_ablation} \textbf{(Left, Middle Left)} we ablate \ourmethod{} with various $\nopt$ on \siga{} with $\tau = -0.53$. In this section, we also provide the same ablation on \valt{} with $\tau = -0.53$ (see \autoref{fig:valt053ablation}), and on both \siga{} and \valt{} with $\tau = -0.4$ (see \autoref{fig:valtsiga04ablation}). 
\begin{figure*}[!ht]
    \begin{center}
        \centerline{\includegraphics[width=0.9\textwidth]{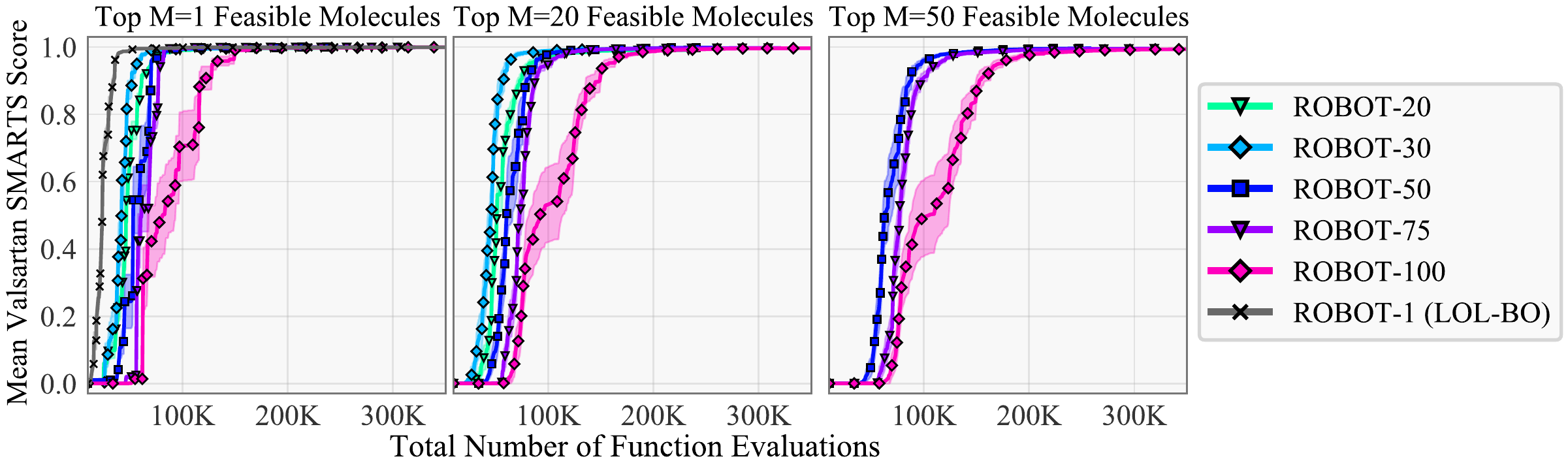}}
      \caption{\valt{} molecular optimization task. Feasible molecules have a maximum fingerprint similarity of 0.53. Ablating \ourmethod{} with various $\nopt$.}
 \label{fig:valt053ablation} 
    \end{center}
\end{figure*}
\begin{figure*}[!ht]
    \begin{center}
        \centerline{\includegraphics[width=0.8\textwidth]{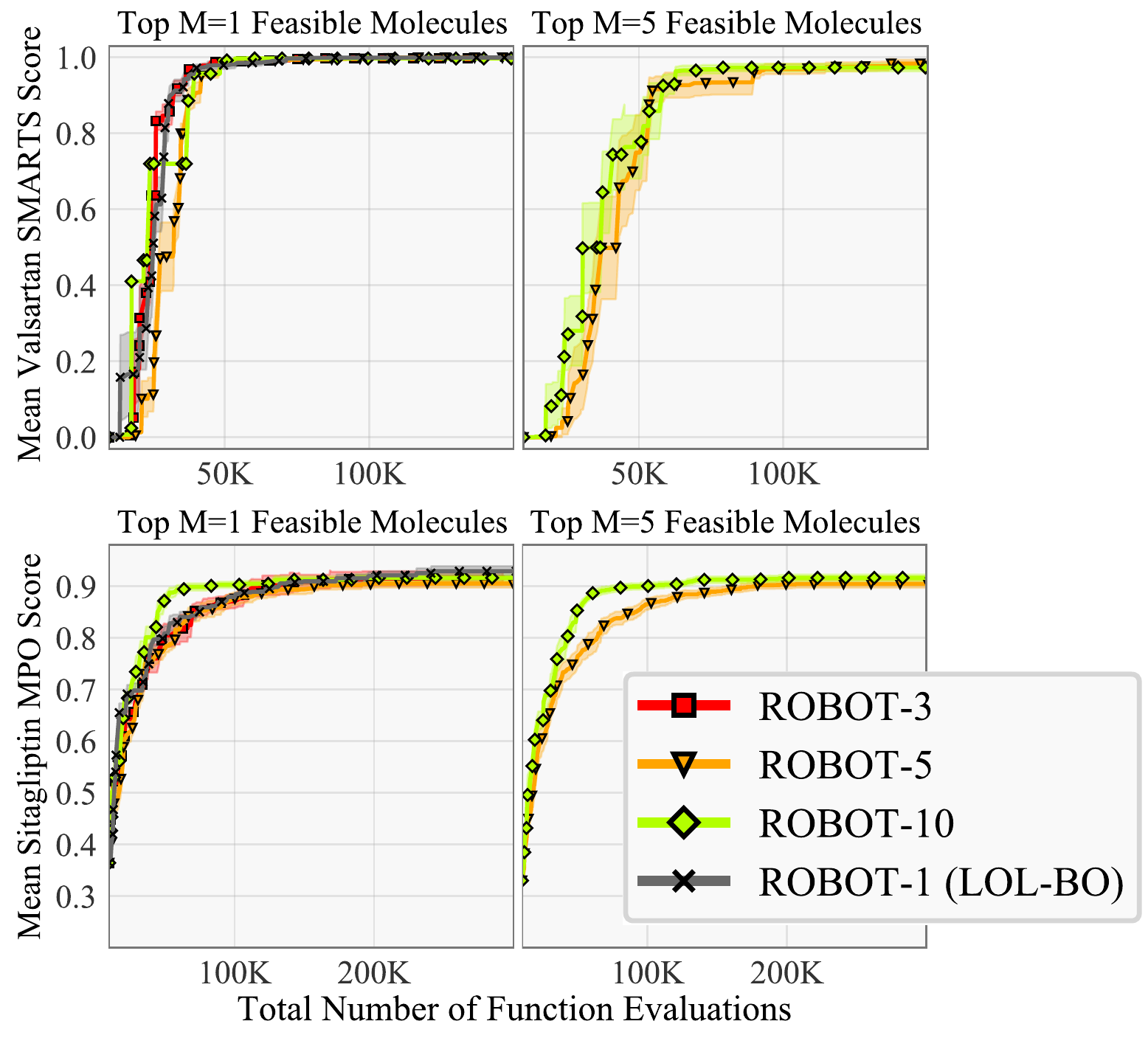}}
        \caption{\valt{} and \siga{} molecular optimization tasks. Feasible molecules have a maximum fingerprint similarity of 0.4. Ablating \ourmethod{} with various $\nopt$.    }
        \label{fig:valtsiga04ablation} 
    \end{center}
\end{figure*}

\subsection{Run-Time Considerations}
\label{sec:runtime} 

Our method scales to millions of evaluations by leveraging minibatch training made possible by the variational GP approximation used by \citet{PPGPR}. 
The total running times for the experiments in this paper is around half a day. 
The one exception to this is the stock portfolio optimization task (\autoref{sec:stocks}) - since this task requires more than 10x as many function evaluations to converge, \ourmethod{} takes roughly one week to run on this task. 

\FloatBarrier
\subsection{Global Consistency of \ourmethod{}}

\label{sec:global_convergence}
In this section we prove that \ourmethod{} converges to the optimal set $S^* = \{\bx_1^*, \ldots, \bx_M^*\}$ as the number of samples tends to infinity. Note that assumptions 4 and 5 correspond to trivial modifications of \ourmethod{} itself, rather than assumptions about the underlying objective. Furthermore, assumption three is a fairly common setting in Bayesian optimization. The primary non-trivial assumption needed in the Theorem below is assumption 2. This assumption requires a certain degree of smoothness from the objective function and the diversity constraint near points in $S^{*}$. In particular, we assume points in an epsilon ball centered at $\bx^{*}_{1}$, $B(\bx^{*}_{1}, \epsilon)$, have higher objective value than points outside the ball. Furthermore, points in $B(\bx^{*}_{2}, \epsilon)$ have higher objective value than other points in the input domain $\inputdom$ that are \textit{feasible} with respect to points in $B(\bx^{*}_{1}, \epsilon)$. Note that assumption 2 is also implied under assumption 1 if both $f(\cdot)$ and $\divf{}(\cdot)$ are continuous and $\divf(\bx^{*}_{1}, \bx^{*}_{2}) > \tau$.

We start by showing the result for $M=2$:

\begin{restatable}[M=2]{theorem}{globalconvergence_m=2}
    \label{theorem_m=2}
    Suppose that \ourmethod{} with default parameters is used under the following assumptions:
    \begin{enumerate}
        \item $S^* = \{\bx_1^*, \bx_2^*\}$ exists and is non-empty. In particular, for each $i$, $\bx^{*}_{i}$ is the unique optimizer to the optimization problem described in \autoref{sec:methods} by \autoref{eq:problem_def}.
        \item There exists an $\epsilon > 0$ such that:
            \begin{enumerate}
                \item For any $\bw_1 \in B(\bx_1^*, \epsilon) \cap \inputdom$ we have $f(\bw_1) > f(\bx)$ for all $\bx \in \inputdom \setminus B(\bx_1^*, \epsilon)$.
                \item For any $\bw_1 \in B(\bx_1^*, \epsilon) \cap \inputdom, \bw_2 \in B(\bx_2^*, \epsilon) \cap \inputdom$ we have $\delta(\bw_1, \bw_2) \geq \tau$.
                \item For any $\bw_1 \in B(\bx_1^*, \epsilon) \cap \inputdom, \bw_2 \in B(\bx_2^*, \epsilon) \cap \inputdom$ we have $f(\bw_2) > f(\bx)$ for all $\bx \in \inputdom \setminus B(\bx_2^*, \epsilon)$ s.t. $\delta(\bx, \bw_1) \geq \tau$.
            \end{enumerate}
        \item The objective $f$ is bounded and the input domain $\inputdom$ is a compact hypercube.
        \item \ourmethod{} generates new initial points when a trust region restarts. These initial points are chosen such that for any $\delta > 0$ and $x \in \inputdom$ there exists $\nu(x, \delta) > 0$ such that the probability that at least one point ends up in a ball centered at $x$ with radius $\delta$ is at least $\nu(x, \delta)$.
        \item \ourmethod{} considers any sampled point an improvement \emph{only if} it improves the current best solution by at least some constant $\gamma > 0$.
    \end{enumerate}
    Then, \ourmethod{} converges to the unique global minimizing set $S^*$.
\end{restatable}

\begin{proof}
We start by observing that each trust region will collect only a finite number of samples before restarting due to conditions (3) and (5) as well as the fact that we similarly to \scbo{} use a finite success and failure tolerance.
This means that each trust region in \ourmethod{} will restart infinitely often with a fresh trust region and hence there is an infinite subsequence of initial points that satisfy (4).
Now, denote the infinite sequence of initial points collected by \ourmethod{} by $\{\hat{\bx}(t)\}_{t \geq 1}$ where $t$ is the iteration number.
We will construct a new sequence $\{\hat{\bx}_1(t)\}_{t \geq 1}$ as follows: 
\[
    \hat{\bx}_1(t) :=
    \begin{cases}
        \hat{\bx}(1) &\text{if } t=1 \\
        \hat{\bx}(t) &\text{if } f(\hat{\bx}(t)) > f(\hat{\bx}_1(t-1)) \\ 
        \hat{\bx}_1(t-1) &\text{otherwise} \\
    \end{cases}      
\]
It now follows directly from the arguments made by \citet{regis2007stochastic} in Theorem 1 and assumption (2) that $\hat{\bx}_1(t) \to \bx_1^*$ almost surely.
Next, assumptions (2) and (4) allows us to find a point $\hat{\bz}_1 \in \{\hat{\bx}_1(t)\}_{t \geq 1}$ that is arbitrarily close to $\bx_1^*$ that is both better than any point outside $B(\bx_1^*, \epsilon)$ and also satisfies the diversity constraint w.r.t. any point $\hat{\bz}_2 \in B(\bx_2^*, \epsilon) \cap \inputdom$. 
We can then construct the following subsequence
\[
    \hat{\bx}_2(t) :=
    \begin{cases}
        \hat{\bx}(1) &\text{if } t=1 \\
        \hat{\bx}(t) &\text{if } \delta(\hat{\bx}(t), \hat{\bz}_1) \geq \tau \text{ and } \delta(\hat{\bx}_2(t - 1), \hat{\bz}_1) < \tau, \\
        \hat{\bx}(t) &\text{if } f(\hat{\bx}(t)) > f(\hat{\bx}_2(t-1)) \text{ and } \delta(\hat{\bx}(t), \hat{\bz}_1) \geq \tau \text{ and } \delta(\hat{\bx}_2(t - 1), \hat{\bz}_1) \geq \tau, \\ 
        \hat{\bx}_2(t-1) &\text{otherwise} \\
    \end{cases}
\]
Following the same argument as before we have that $\hat{\bx}_2(t) \to \bx_2^*$ almost surely. 
\end{proof}

We can extend these ideas to cover any finite $M$ by extending these assumptions.

\begin{restatable}[]{corollary}{globalconvergence}
    Assume 3-5 from Theorem \ref{theorem_m=2} are satisfied. In addition, we also assume that the following is true:
    \begin{enumerate}
        \item $S^* = \{\bx_1^*, \ldots, \bx_M^*\}$ and is non-empty, where $M > 2$ is finite.
        \item There exists an $\epsilon > 0$ such that:
            \begin{enumerate}
                \item For any $\bw_1 \in B(\bx_1^*, \epsilon) \cap \inputdom$ we have $f(\bw_1) > f(\bx)$ for all $\bx \in \inputdom \setminus B(\bx_1^*, \epsilon)$.
                \item  The following holds for $j=2, \ldots, M$: For any $\bw_1 \in B(\bx_1^*, \epsilon) \cap \inputdom, \ldots, \bw_j \in B(\bx_j^*, \epsilon) \cap \inputdom$ we have $\min_{i\in \{1, \ldots, j-1\}} \delta(\bw_i, \bw_j) \geq \tau$.
                \item  The following holds for $j=2, \ldots, M$: For any $\bw_1 \in B(\bx_i^*, \epsilon) \cap \inputdom, \ldots, \bw_j \in B(\bx_j^*, \epsilon) \cap \inputdom$ we have $f(\bw_j) > f(\bx)$ for all $\bx \in \inputdom \setminus B(\bx_j^*, \epsilon)$ s.t. $\min_{i\in \{1, \ldots, j-1\}} \delta(\bx, \bw_i) \geq \tau$.
            \end{enumerate}
    \end{enumerate}
    Then \ourmethod{} converges to the unique set $S^*$. 
\end{restatable}

\subsection{Limitations and Assumptions of \ourmethod{} }
\label{sec:limitations}
Similar to any Bayesian optimization method, \ourmethod{} assumes that the probabilistic surrogate model can obtain a good fit to the black box function.
Additionally, \ourmethod{} assumes that the diversity function $\divf{}$ is cheap to compute relative to the black box function.

\subsubsection{Selection of $\tau$}
\label{sec:tau}
\ourmethod{} assumes that the diversity constraints are mild enough and $M$ is small enough that finding $M$ diverse points in the search space is possible. 
We suggest choosing a $\tau$ value which represents the most relaxed constraint possible that still enforces enough diversity to be meaningful to the practitioner. 
We assume that the practitioner can evaluate the diversity function to determine what values indicate sufficient diversity for their particular application. 

\subsection{Extending \turbo{} to \ourmethod{}}
\label{sec:robot-vs-turbo}
In this section, we describe the modifications necessary to modify an implementation of \turbo{} into an implementation of \ourmethod{}. 
Note that a full BoTorch implementation of \ourmethod{} is also available, as linked to in the experimental results section of the main text. 
We give a pseudocode algorithm for \ourmethod{} (\autoref{alg:algo}). 
Lines of the algorithm in \textcolor{blue}{blue} show the parts of \ourmethod{} that differ from \turbo-$M${}.
\begin{algorithm}
	\caption{\ourmethod{} Algorithm}
	Inputs: f, $D_0$, M, \textcolor{blue}{$\delta$, $\tau$} 
	\label{alg:algo}
	\begin{algorithmic}[1]
	    \For {$i = 1, ..., M$}
	    \State Initialize trust region $\tr{_i}$ 
	    \EndFor
		\For {Every time step t}
		\State Update surrogate model on $D_t$
		\State
		\textcolor{blue}{ $S^{+}_{t}=\{\bx'^{(t)}_{1},...,\bx'^{(t)}_{\nopt}\}$, where:
\begin{align}
    \bx'^{(t)}_{1} &= \argmax_{(\bx,y) \in D_{t}} y \nonumber \\
    \bx'^{(t)}_{i} &= \argmax_{(\bx,y) \in D_{t}} y \;\; \mathrm{s.t.} \;\; \forall j < i \;\; \divf{}(\bx, \bx'^{(t)}_{j}) \geq \tau 
\end{align}
}
        \State \textcolor{blue}{Set center of trust region $\tr{i}$ to $\bx'^{(t)}_{i}$
        }
		\For {$i = 1, ..., M$}
	    \State Select candidate $x_i$ from $\tr{_i}$ using acquisition function 
	    \EndFor
	    \For {$i = 1, ..., M$}
	    \If{\textcolor{blue}{$\delta$($x_i$, $x_j$) $< \tau$ for any $j < i$}}
	    \State \textcolor{blue}{Discard $x_i$}
	    \Else{}
	    \State $y_i = f(x_i)$
	    \State Update length of $\tr{_i}$ based on $y_i$
	    \EndIf 
	    \EndFor
		\State $D_{t+1}v\leftarrow D_{t} \cup (\bx, \by)$
	    \EndFor
	\end{algorithmic} 
\end{algorithm} 
We provide additional clarification on the differences highlighted between \ourmethod{} and \turbo{}-$M$ here. Both methods use $M$ hyper-rectangular trust regions and dynamically update the length and center of each trust region using the dynamics of \cite{TuRBO}. However, \ourmethod{} performs acquisition to find a set of $M$ diverse solutions according to a diversity function $\divf{}$ and threshold $\tau$, with each trust region responsible for one solution, while \turbo{}-$M$ finds a single solution from across the trust regions. Accordingly, \turbo{}-$M$ allocates budget related to the strength of each trust region's incumbent, and some trust regions may see very few allocations if they start in poor regions of the search space. 

In contrast, \ourmethod{} selects the same number of candidates from each trust region on each step of optimization, maintains a rank-ordering of the $M$ trust regions and re-centers the $M$ trust regions in rank-order on best diverse set of points from the entire shared data history $S^{+}_{t}=\{\bx'^{(t)}_{1},...,\bx'^{(t)}_{\nopt}\}$. \turbo{}-$M$ re-centers each trust region on the best candidate queried by that individual trust region. Finally, \ourmethod{} discards candidates from lower-ranking trust regions if they are not sufficiently diverse from the candidates from higher-ranking trust regions while \turbo{}-$M$ does not discard candidates. 

\thispagestyle{empty}

\end{document}